\documentclass{article}

\usepackage{microtype}
\usepackage{graphicx}
\usepackage{booktabs} %

\usepackage{hyperref}

\usepackage[accepted]{icml2021}

\usepackage{framed,multirow}

\usepackage{amssymb}
\usepackage{latexsym}
\usepackage{amsthm}%

\usepackage{url}
\usepackage{xcolor}
\definecolor{newcolor}{rgb}{.8,.349,.1}

\usepackage[titletoc]{appendix}
\usepackage{url}

\usepackage{amsmath,amssymb}
\usepackage[utf8]{inputenc}
\usepackage{graphicx}%
\usepackage{braket}
\usepackage{tikz}
\usetikzlibrary{automata,arrows,arrows.meta,positioning,calc,shapes,trees}

\usepackage{bbm}%
\usepackage{relsize}
\usepackage{float}
\floatstyle{plaintop}%
\restylefloat{table}
\usepackage{pgffor}

\usepackage[space]{grffile} %

\usepackage{datatool}

\usepackage{csvsimple}
\usepackage{xspace}
\usepackage{etoolbox} %
\usepackage{mathtools} %

\usepackage{enumitem} %

\usepackage{xargs}                      %
\usepackage{hyperref}
\hypersetup{
    colorlinks,
    citecolor=black,
    filecolor=black,
    linkcolor=black,
    urlcolor=black
}

\usepackage{xhfill}

\usepackage[breakable, theorems, skins]{tcolorbox}

\usepackage[normalem]{ulem} %

\newtheorem*{rep@theorem}{\rep@title}
\newcommand{\newreptheorem}[2]{%
	\newenvironment{rep#1}[1]{%
		\def\rep@title{#2 \ref{##1}}%
		\begin{rep@theorem}}%
		{\end{rep@theorem}}}
\makeatother
\newreptheorem{theorem}{Theorem}
\newreptheorem{lemma}{Lemma}
\newreptheorem{corollary}{Corollary}

\newtheorem{theorem}{Theorem}
\newtheorem{corollary}{Corollary}[theorem]
\newtheorem{lemma}[theorem]{Lemma}

\theoremstyle{definition}

\newcommand{\abs}[1]{\left\lvert#1\right\rvert}

\makeatletter
\newcommand{\distas}[1]{\mathbin{\overset{#1}{\kern\z@\sim}}}%
\newsavebox{\mybox}\newsavebox{\mysim}
\newcommand{\distras}[1]{%
  \savebox{\mybox}{\hbox{\kern3pt$\scriptstyle#1$\kern3pt}}%
  \savebox{\mysim}{\hbox{$\sim$}}%
  \mathbin{\overset{#1}{\kern\z@\resizebox{\wd\mybox}{\ht\mysim}{$\sim$}}}%
}
\makeatother

\DeclareMathOperator{\EX}{\mathbb{E}}%

\newcommand\independent{\protect\mathpalette{\protect\independenT}{\perp}}
\def\independenT#1#2{\mathrel{\rlap{$#1#2$}\mkern2mu{#1#2}}}

\DeclarePairedDelimiterX{\infdivx}[2]{(}{)}{%
  #1\;\delimsize\|\;#2%
}

\newcommand{\KLdiv}{\text{KL}\infdivx}

\newcommand{\R}{\mathbb{R}}
\newcommand{\E}{\mathbb{E}}

\icmltitlerunning{Rate-Distortion Analysis of Minimum Excess Risk in Bayesian Learning}

\begin{document}

\twocolumn[
\icmltitle{Rate-Distortion Analysis of Minimum Excess Risk in Bayesian Learning}

\begin{icmlauthorlist}
\icmlauthor{Hassan Hafez-Kolahi}{ce}
\icmlauthor{Behrad Moniri}{ee}
\icmlauthor{Shohreh Kasaei}{ce}
\icmlauthor{Mahdieh Soleymani Baghshah}{ce}
\end{icmlauthorlist}

\icmlaffiliation{ce}{Department of Computer Engineering,
	Sharif University of Technology, Tehran, Iran}
\icmlaffiliation{ee}{Department of Electrical Engineering,
	Sharif University of Technology, Tehran, Iran}

\icmlcorrespondingauthor{Shohreh Kasaei}{kasaei@sharif.edu}

\icmlkeywords{Compressed Representation; Generalization Bound;  Information Bottleneck}

\vskip 0.3in
]

\printAffiliationsAndNotice{}  %

\begin{abstract}
In parametric Bayesian learning, a prior is assumed on the parameter $W$ which determines the distribution of samples. In this setting, Minimum Excess Risk (MER) is defined as the difference between the minimum expected loss achievable when learning from data and the minimum expected loss that could be achieved if $W$ was observed. In this paper, we build upon and extend the recent results of \cite{xu2020minimum} to analyze the MER in Bayesian learning and derive information-theoretic bounds on it. We formulate the problem as a (constrained) rate-distortion optimization and show how the solution can be bounded above and below by two other rate-distortion functions that are easier to study. The lower bound represents the minimum possible excess risk achievable by \emph{any} process using $R$ bits of information from the parameter $W$. For the upper bound, the optimization is further constrained to use $R$ bits from the training set, a setting which relates MER to information-theoretic bounds on the generalization gap in frequentist learning. We derive information-theoretic bounds on the difference between these upper and lower bounds and show that they can provide order-wise tight rates for MER under certain conditions. This analysis gives more insight into the information-theoretic nature of Bayesian learning as well as providing novel bounds. 
\end{abstract}

\section{Introduction}
\label{sec:introduction}
One of the main problems studied in statistical learning theory is the
excess risks of learning algorithms, which is the gap between the achieved error and the best possible error if the distribution was known \cite{lecam1973convergence, assouad1983deux,Keener2006}. An interesting question in this regard is to study lower bounds on the 
excess risk which could be achieved by any algorithm. This concept is usually studied in the frequentist setting, in which a family of distributions is assumed and minimax bounds are derived to study if an algorithm which only has access to $n$ samples from the distribution,  can work well for all the distributions in the family.

Recently, \cite{xu2020minimum}
proposed a framework to define and study Minimum Excess Risk (MER) in Bayesian learning. In the Bayesian learning, it is assumed that the underlying distribution is described by a variable $W\in\mathcal{W}$ and a prior $P_W$ is considered which describes the probability of any $W$ before observing data.  The joint distribution of $W$, the training set $Z^n=\{(X_i,Y_i)\}_{i=1}^n \in (\mathcal{X} \times \mathcal{Y})^n$, and a test sample $Z=(X,Y)\in \mathcal{X}\times\mathcal{Y}$, is described as $P_W \otimes{(P_{Z}^W)^n}\otimes P_Z^W$. The goal is to find a function $\hat{h}:\mathcal{X}\to\mathcal{Y}$ after observing the training set,  in a way that $\EX[\ell(Y,\hat{h}(X))]$ is small, where $\ell:\mathcal{Y}\times\mathcal{Y} \to \mathbb{R}$ is the loss function. 

In order to quantify the hardness of a problem,  \cite{xu2020minimum} define Minimum Excess Risk as the gap between the expected error of the best algorithm which only has access to the data and the minimum expected error if $W$ was also observed. They show that for a variety of loss functions, the conditional mutual information $I(W;Y|Z^n,X)$ appears in the upper bounds on MER.
When $W\in\mathbb{R}^p$, using information-theoretic results on the rate of $I(W;Z^n)$, they achieve bound of $O(\sqrt{p/n})$ on MER as $n \to \infty$, given that some assumptions on the distribution and loss hold 
(see Section \ref{sec:it_bounds_on_MER}). 
They also show that the bound can be improved to $O(p/n)$ for two specific losses: logarithmic loss and quadratic loss (for bounded $\mathcal{Y}$). They left the study of lower bounds on MER as an open problem. 

In this work, we adapt a source coding view on learning and introduce a (variant of) rate-distortion optimization which captures the notion of MER. 
Then, we demonstrate how the constraint on this rate-distortion minimization can \emph{naturally} be weakened and strengthened to achieve lower and upper bounds, respectively. These lower and upper bounds are easier to study and we derive a variety of results on them. In particular, we study the lower bound with tools from the source coding theory, and demonstrate that under some conditions, MER is lower bounded by $\Omega(p/n)$. This concludes the rate analysis of MER for the cases in which the matching upper bound $O(p/n)$ exists and show that both upper and lower bounds are order-wise tight as $n \to \infty$. As an important example, we show that the bounds are order-wise tight for quadratic loss when $\mathcal{Y}$ is bounded and the distribution is suitably \emph{smooth} (see Section~\ref{sec:applications}).

The studied rate distortion problems (original, upper bound, and lower bound) all have interesting interpretations and might be of interest by themselves.

\subsection{An Appetizer for the Rate-Distortion View}
\label{sec:appetizer}

Loosely speaking, the main idea behind the rate-distortion view developed in this paper is as follows. The variable $W$ is first generated and the dataset $Z^n$ is generated from $W$. Our goal is to observe $Z^n$ and find $\hat{h}(x)$ which performs well compared to the case where $W$ is known.
Thus, if we could decode $W$ from $Z^n$ perfectly, an MER equal to zero would be achieved. However,
in almost all applications of interest, 
it is impossible to find the exact value of $W$, since the information we can extract from $W$ to build $\hat{h}$, is bounded above by $I(W;Z^n)$. In particular, if $W$ is continuous, an infinite number of samples are needed for its exact decoding.
But of course, we don't need full recovery of $W$ to get a \emph{good enough} $\hat{h}$. 
So the question is how good we can act, based on a suitable distortion function, if only $I(W;Z^n)$ nats of information about $W$ is received \emph{through $Z^n$}. 

This line of reasoning makes it natural to study the problem as a rate-distortion optimization. %
But there is a challenge in using rate-distortion theory to study the learning problem: in an standard rate-distortion problem, we can decide how we encode $W$, but in learning problems, a (random) preprocess $W\to Z^n$ is also enforced. If we remove this constraint, we will have a standard rate-distortion problem.
Since the feasible set is enlarged, this gives us a lower bound on the original minimization. 
It is not obvious how efficient this preprocess acts; i.e., if one is asked to use $R=I(W;Z^n)$ nats to represent $W$ by an intermediate variable $\Xi$ (in an arbitrary space of choice) in a way that it is possible to recover a good $\hat{h}$, is it a good idea to just generate $n$ i.i.d. samples from $P_{XY}^W$; i.e., use $\Xi=Z^n$? 
We will try to answer such questions by quantifying and studying lower bounds.

On the other hand, there is an interesting question which is answered by studying an upper bound on the rate-distortion function. 
If we know that only $R=I(W;Z^n)$ nats about $W$ are present in the dataset $Z^n$, we might hope to be able to just extract those nats and don't rely on $Z^n$ more than necessary. 
To quantify this idea and study it, we can also restrict $I(Z^n; \hat{h})\le R = I(W;Z^n)$ and see if we can still find an $\hat{h}$ with a good performance? This scenario is similar to the frequentist approach of model compression in which 
the mutual information between the training set and the learned model 
is restricted to control the generalization gap. 

To better understand the information-theoretic properties of learning, we will study these rate-distortion functions and derive information-theoretic bounds on their difference. %
In particular we will show that (under some smoothness conditions) all three rate-distortion functions converge as $n\to\infty$. The rates of convergence are also derived which shows that the bounds are order-wise tight for quadratic loss under certain conditions. We also provide non-asymptotic information-theoretic bounds which explain the difference between these rate-distortion functions for finite samples.

While rate-distortion theory was used before in learning theoretic settings (see Section~\ref{sec:previous_work}), the systematic view developed in this paper as well as the derived bounds are novel to the best of our knowledge.

\subsection{Outline of the Paper}
In Section~\ref{sec:previous_work}, the related literature is discussed. The notations are introduced in Section~\ref{sec:notation}. Section~\ref{sec:it_bounds_on_MER} is devoted to information-theoretic upper bounds on MER. In Section~\ref{sec:rd_analysis_of_MER}, the main results of the paper on the rate-distortion view of MER are presented. In Section~\ref{sec:applications}, some applications of the developed tools are studied. Finally, conclusions and future works are presented in Section~\ref{sec:conclusion}. The proofs of theorems are presented in the supplementary materials.

\section{Related Work}
\label{sec:previous_work}
Deriving minimax bounds on excess risk in the frequentist setting is a well studied problem.
The Le Cam's and Assouad's methods are two of the most widely used approaches for deriving lower bounds on minimax risk \cite{lecam1973convergence, assouad1983deux}. Fano’s method is also a popular method based on Fano’s lower bound on the error probability in an M-ary testing problem \cite{yang1999information}. 

In Bayesian learning, one of the tracks which is related to MER, 
is the convergence of posterior to the true parameter
\cite{ghosal2000convergence,shen2001rates,ghosal2007convergence,le2012asymptotics}. The main difference between this line of works and MER studied by \cite{xu2020minimum} is that  the former tries to analyze Bayesian inference from a frequentist perspective, while in latter, $W$ is still considered as random in the analysis. The information-theoretic results used in this new setting as well as the definition of MER itself (which is based on an expectation on $P_W$) are
 influenced by this view on the problem. 
Moreover, the subject of study in previous works is usually estimation of the parameter while in MER, the Bayes risk for the random variable of interest is directly studied. Results from convergence of posterior are useful to derive bounds on MER (e.g. see Theorem~\ref{theorem:convergence_of_rate_distortion_functions} of current paper as well as Section~4 of \cite{xu2020minimum}).

These recent results of \citet{xu2020minimum} can be seen as extensions to the
universal prediction of \cite{merhav1998universal}.
In universal prediction, the accumulated loss on a sequence of samples is studied (using an approach similar to universal source coding). In contrast, this new treatment allows the analysis of the supervised setting with general subgaussian loss. 
Moreover, they utilize a refined treatment which yields direct bounds on the error of estimating the label of the test sample (the last sample) when the training set is given. 

Another track which heavily influences the current work, is the recent series of results in frequentist learning which use mutual information between the learned model and the dataset 
to control the generalization gap \cite{RussoHowmuchdoes2015,RussoControllingbiasadaptive2016,XuInformationtheoreticanalysisgeneralization2017, BassilyLearnersthatUse2018,AsadiChainingMutualInformation2018,steinke_reasoning_2020,hafez2020conditioning}. While the setting in these works is different, the mathematical tools developed to derive information-theoretic bounds are similar. In particular, we use ideas from  \cite{steinke_reasoning_2020} to derive new bounds on MER. Moreover, the upper bound on rate-distortion function studied in this paper, which is based on $I(Z^n;\hat{h})$, is directly related to this frequentist setting (see Section \ref{sec:upper_bound}). 
A relevant work on this setting is \cite{bu2020information} which used model compression to produce $\tilde{h}$ from $\hat{h}$ to control the generalization gap.

In \cite{gao2019rate}, a rate-distortion optimization is utilized to understand the limits of model compression, and results for linear models are derived. Their setting is similar to a loosened variant of the rate-distortion lower bound studied in this paper where the solution is restricted to be from the parametric family (see Section~\ref{sec:applications}).

In \cite{nokleby2016rate}, upper bounds on the expected excess risk are derived for certain class of problems which satisfy a notion of ``interpolation set". Their Bayesian treatment of the parameter makes their results comparable to the setting of \cite{xu2020minimum} when studying zero-one loss.

\section{Notation and Preliminaries}
\label{sec:notation}
Random variables and their realizations are represented with uppercase and lowercase letters respectively; e.g., $x\in\mathcal{X}$ is a realization of random variable $X$.  Conditional distributions and expectations are identified with superscripts; e.g., $P_X^z$ indicates the conditional distribution of $X$ given $Z=z$ and $\EX_X^z[f(X,z)]$ indicate the expectation of $f(X,z)$ based on this distribution.
$\KLdiv{P_X}{Q_X}=\int \log \frac{P(X)}{Q(X)}dP$ is the KL divergence of distribution $P_X$ from $Q_X$. Mutual information is defined as $I(X;Y)=\KLdiv{P_{XY}}{P_X\otimes{P_Y}}$ where $P_X$ and $P_Y$ are marginal distributions of $P_{XY}$. 
The conditional mutual information is defined as $I(X;Y|Z)=\EX_Z[I^Z(X;Y)]$ in which for all $z$, $I^z(X;Y)$ is the mutual information on the conditioned distributions $P_{XY}^z$, i.e. $I^z(X;Y)=\KLdiv{P_{XY}^z}{P_X^z\otimes P_Y^z}$.
Throughout the paper, all logarithms are in natural base and all information-theoretic quantities are in nats.

Given a distribution $P_{XY}$ on a set $\mathcal{X}\times\mathcal{Y}$ 
and a loss function 
$\ell:\mathcal{Y}\times \mathcal{Y}\to \mathbb{R}$,
the Bayes risk of estimating $Y$ from $X$ is denoted by 
$$R_{\ell}(Y|X)= \inf_{\psi:\mathcal{X}\to\mathcal{Y}} \EX[\ell(Y,\psi(X)].$$  
It is assumed that the infimum is attained, and the optimal decision for a given $x$ is represented by $\psi^*_{\ell,Y|X}(x)$, where we omit $\ell$, if it is clear from the context.

\section{MER and the Information-Theoretic Upper Bounds}
\label{sec:it_bounds_on_MER}

As described in Section \ref{sec:introduction}, a common scenario in Bayesian learning is to consider a prior distribution $P_W$ on $W$ and consider the joint distribution $P_W\otimes (P_{XY}^W)^n \otimes P_{XY}^W$ which generates $W,Z^n,Z$. Here, $Z^n=\{(X_i,Y_i)\}_{i=1}^n$ is the training set and $Z=(X,Y)$ is the test sample. Usually, it is also assumed that $X$ is independent of $W$, and we have the distribution $P_{XY}^W =P_X \otimes P_{Y}^{XW}$, i.e., the unknown parameter is just used in describing the relation between $X$ and $Y$. The goal is to predict $Y$ when $Z^n$ and $X$ are given. The Bayes risk for this task is 
$R_\ell(Y|Z^n,X)$. If the parameter $W$ was known, we could do better and achieve $R_\ell(Y|W,X)$. MER is defined as the expected extra price we should pay as a result of not knowing $W$; i.e.,
\begin{equation}
    \text{MER}_\ell ^n = R_\ell(Y|Z^n,X) - R_\ell(Y|W, X).
\end{equation}

Upper bounds on  MER for various loss functions are studied in \cite{xu2020minimum}.
Note that if we have a Markov chain $Y - U - V$ then there is a data processing inequality for Bayes risk; i.e., 
$$R_\ell(Y|U)\le R_\ell(Y|V)$$ (see Lemma~1 of \cite{xu2020minimum}).
The following lemma, which is due to 
Theorem 4 of \cite{xu2020minimum}, gives an upper bound on the looseness of this inequality 
when the loss function is bounded.  

\begin{lemma}
\label{lemma:base_bound_for_Bayes_data_process_gap}
Consider random variables $Y,U$ and $V$ forming Markov chain $Y - U - V$ and an arbitrary non-negative bounded function $\ell:\mathcal{Y}\times \mathcal{Y} \to [0,b]$. We have
\begin{equation}
R_\ell(Y|V) - R_\ell(Y|U) \le \sqrt{\frac{b^2}{2}I(Y;U|V)}.
\end{equation}
\end{lemma}

Using this lemma, it is straightforward to derive upper bounds on $\text{MER}_\ell^n$ as
\begin{align}
\text{MER}_\ell^n&=
R_\ell(Y|Z^n,X)-R_\ell(Y|W,X) \nonumber \\
&= R_\ell(Y|Z^n,X)-R_\ell(Y|W,Z^n,X)
\nonumber \\ &\le \sqrt{\frac{b^2}{2} I(Y;W|Z^n,X) }
\label{eq:MER_bound_conditional}
\\ &\le \sqrt{\frac{b^2}{2n} I(W;Z^n) },
\label{eq:MER_bound_MI_W_dataset}
\end{align}
where the second equality is due 
to the fact that  $Y\independent Z^n| W,X$.
The final inequality is proved 
by noting that $I(Z;W|Z^n)$ is a decreasing function of $n$ and applying chain rule on $I(Z^n;W)$ (see the Proof of Theorem 2 in \cite{xu2020minimum}).
Note that the bound (\ref{eq:MER_bound_conditional}) could be much better than (\ref{eq:MER_bound_MI_W_dataset}). The main reason is that it does not depend on $I(W;X)$ which could be large.
The improvement one can achieve when using (\ref{eq:MER_bound_conditional}) instead of (\ref{eq:MER_bound_MI_W_dataset}) is similar to using the \emph{conditioning} technique to improve the information-theoretic generalization bounds in frequentist learning \cite{hafez2020conditioning}. Actually, the same conditioning technique 
is at the heart of deriving the first bound (see proof of Lemma \ref{lemma:base_bound_for_Bayes_data_process_gap} in the supplementary materials).
It is also worth noting that if 
the distribution on $W$ is not known, but the capacity of channel $P_{Z^n}^W$ is limited, then $I(W;Z^n)$ is controlled and Equation~\eqref{eq:MER_bound_MI_W_dataset} can be used to derive a redundancy-capacity result similar to universal prediction \cite{merhav1998universal}.

The following lemma can be used 
along Lemma~\ref{lemma:base_bound_for_Bayes_data_process_gap} to achieve convergence rates. This is a classic result on growth rate of mutual information between observations and the parameter which can be found in \cite{clarke1990bayes, clarke1994jeffrey}.

\begin{lemma}
	\label{lemma:fisher}
	If $W$ is taking values in a $p$-dimensional compact subspace of $\R^p$, and the model $P_{Z}^w$ is smooth in $w$, then as $n \to \infty$, we have
	\begin{equation*}
	I(W; Z^n) = \frac{p}{2} \log\big(\frac{n}{2\pi e}\big) + h(W) + \frac{\E \big[\log| J_{Z}^W (W)|\big]}{2} + o(1),
	\end{equation*}
	in which $|J_{Z}^W(w)|$ is the determinant of the Fisher information matrix about $W$ contained in $Z$.	Rigorous statement of the smoothness conditions can be found in the appendix.
\end{lemma}

Using this lemma, it can be shown that 
$I(Y;W|Z^n) = O(1 / n)$ as $n \to \infty$.
This gives us a rate of $O(\sqrt{1 /n})$ on $\text{MER}_\ell^n$ for any bounded loss. In \cite{xu2020minimum}, it is also proved that for bounded quadratic loss and logarithmic loss the square root can be removed which improves the rate to $O(1/n)$.

\subsection{Dropping the Square Root}
\label{sec:dropping_the_square_root}
Whether it is possible to drop the square root for a general bounded loss is an open problem. In this section we demonstrate that this is possible for the realizable case. %

\begin{lemma}
\label{lemma:bound_for_Bayes_data_process_gap_for_realizable}
Consider random variables $Y, U$, and $V$ forming Markov chain $Y - U - V$ and an arbitrary non-negative bounded function $\ell:\mathcal{Y}\times \mathcal{Y} \to [0,b]$. We have 
\begin{equation}
R_\ell(Y|V)  \le 2 R_\ell(Y|U)  +3bI(Y;U|V).
\end{equation}
\end{lemma}
To prove this bound, a symmetrization technique that is used in \cite{steinke_reasoning_2020} to derive a variety of bounds on generalization gap, is adapted. 
For the case where $R_\ell({Y|U})$ is close to zero, this bound can give better results compared  to Lemma~\ref{lemma:base_bound_for_Bayes_data_process_gap}. 
The mutual information $I(Y;U|V)$ can be unbounded in certain problems. In particular, in realizable setting, this can happen when the random variables are continuous and the relation between them is deterministic. 
Informally, this is due to the fact that $I(Y;U|V)$ quantifies the nats necessary for full recovery of $Y$ (which could be unbounded for continuous random variables). 
This can be solved by covering the space of $\mathcal{Y}$ at different levels and adopting the chaining technique to acquire sharper bounds (e.g., see \cite{AsadiChainingMutualInformation2018} for an application of the chaining technique on information-theoretic generalization bounds). This is discussed in the supplementary materials. %

\section{Rate-Distortion Analysis of MER} \label{sec:rd_analysis_of_MER}
Inequalities (\ref{eq:MER_bound_MI_W_dataset}) and (\ref{eq:MER_bound_conditional}) give us information-theoretic \emph{upper} bounds on MER. Lower bounding MER in Bayesian learning has been remained as an open problem \cite{xu2020minimum}. The tools we develop in this section let us study the lower bounds as well.

\subsection{The Challenge of Finding a Lower Bound} %
First, it should be noted that it is not possible to have a matching lower bound in the form of $\mathrm{(MER)}_\ell^n \geq \alpha \sqrt{I(Y;W|Z^n,X)}$ for some $\alpha>0$. Loosely speaking, the reason is that it is possible that $Y$ and $W$ share many bits but those bits are not used in the loss function. In other words, the information contained in $W$ about $Y$ is not necessarily related to loss.
As an example, consider the toy problem where $\mathcal{W}=\mathcal{Y}=[-2,2]^2$, and
\begin{equation*}
    \begin{cases}
        W_i\distas{\text{i.i.d.}} \mathrm{Unif}(-1,1),\\
        \epsilon_i\distas{\text{indep.}} \mathrm{Unif}(-a_i,a_i),\\
        Y=(Y_1,Y_2)=(W_1+\epsilon_1, W_2+\epsilon_2).
    \end{cases}
\end{equation*}
Define $\ell((y_1,y_2), (\hat{y}_1,\hat{y}_2))=c_1 (y_1-\hat{y}_1)^2 + c_2 (y_2-\hat{y}_2)^2$. Here, $a_1,a_2, c_1$, and $c_2$ are hyper-parameters defining the problem.   Now consider the extreme case where $a_1=0, a_2=1, c_1=1$, and $c_2=0$. In this case, the loss function is ignoring the second dimension of $Y$, which is the harder one to estimate. Actually, by observing a single sample, $W_1$ is found and $\forall n>1, \mathrm{MER}_\ell^n = 0$. %
However, the mutual information $I(W;(Y_1,Y_2)|Z^n)\ge I(W;Y_2|Z^n)$ approaches zero only as $n\to \infty$.%

Based on such observations, we argue that in order to derive MER lower bounds, the relation between 
the used rate and the loss function $\ell$ should be considered more carefully. This was one of the main motivations to define the MER as a rate-distortion problem. 
But, it is also insightful in itself to study Bayesian learning from a source coding perspective, as was discussed in Section~\ref{sec:appetizer}.

\subsection{Rate-Distortion Optimization}
\label{sec:subsection_rate_distortion_optimization}

Rate-Distortion theory was introduced by \cite{Shannonmathematicaltheorycommunication1948,shannon1959coding} to quantify the minimum average number of bits needed to transmit a random variable with a given maximum distortion. Let $P_X$ be a distribution over $\mathcal{X}$ and $X^n = \{X_i\}_{i = 1}^n$ be $n$ i.i.d. samples of $P_X$. An encoder $f_n: \mathcal{X}^n \to \{1, 2, \dots, 2^{nR}\}$ maps the message into a codeword, and the decoder $g_n: \{1, 2, \dots, 2^{nR}\} \to \hat{\mathcal{X}}^n$, decodes the codeword.  The distortion function $d: \mathcal{X}\times\hat{\mathcal{X}} \to \R^+$, measures the distortion and $d(X^n, \hat{X}^n)$ is the average distortion of $X_i$ and $\hat{X}_i$s. For a given rate $R$, the rate-distortion function $D(R)$ is the infimum of all distortions $D$, such that there exists a sequence $(f_n, g_n)$ with codeword size $2^{nR}$, that $\lim_{n\to \infty} \E[d(X^n, g_n(f_n(X^n))] \leq D$. It is shown that
\begin{align*}
    D(R)= \inf_{P_{\hat{X}}^{X}} \EX[d(X,\hat{X})]
 \;\; \mathrm{s.t.}\;\; I(X; \hat{X})\le R. 
\end{align*}
We denote this optimization as the rate-distortion minimization, and the function $D(R)$ as the the rate-distortion function (some authors call this the distortion-rate function to contrast with another function $R(D)$ which maps distortion to rate). For an overview of the classic rate distortion theory, see Chapter 10 of \cite{CoverElementsinformationtheory2012}. 

Now, we are ready to precisely define the rate-distortion problem describing the MER.
To do so, let's define the distortion function as the excess risk of $\hat{h}$ compared to the Bayes decision $h^*_w(x)$, i.e.
\begin{equation}
\label{eq:d-def}
d(w,\hat{h}) = \EX_{XY}^w [\ell(Y,\hat{h}(X))-\ell(Y,h^*_w(X))].
\end{equation}
Note that this definition is consistent with our final goal which is to study MER: if we consider the optimal learning algorithm which generates $\hat{h}(.)= \psi^*_{Y|Z^nX} (z^n, .)$ for any given dataset $Z^n=z^n$, the expected distortion is 
\begin{align}
\EX_{WZ^n}[&d(W,\psi^*_{Y|Z^nX}(Z^n, \cdot \,))] \nonumber\\
&= \EX_{WZ^nXY}[\ell(Y,\psi^*_{Y|Z^nX} (Z^n, X))
\nonumber\\&\hspace{70pt}-\ell(Y,\psi^*_{Y|WX}(W,X))]
\nonumber\\ &= R_\ell(Y|Z^n,X) - R_\ell(Y|W,X)
\nonumber\\ & = \text{MER}_\ell^n.
\label{eq:expected_distortion_eq_MER}
\end{align}

Now, we define the (constrained) rate-distortion optimization as
\begin{align}
\label{eq:rate_distortion_base}
    D_n(R)= \inf_{P_{\hat{h}}^{Z^n}} \EX[d(W,\hat{h})],
\\
\mathrm{s.t.}\; I(W;\hat{h})\le R, \nonumber
\end{align}
in which the expectation and mutual information are evaluated with respect to $P_{W\hat{h}}$ which is the marginal distribution of $P_W\otimes{P_{Z^n}^W}\otimes{P_{\hat{h}}^{Z^n}}$. Note that in standard rate-distortion problems, we are allowed to directly optimize $P_{\hat{h}}^W$. However, here there is an extra constraint that $P_{\hat{h}}^W={P_{Z^n}^W}\otimes{P_{\hat{h}}^{Z^n}}$.
Also note the dependence of $D_n(R)$ on $n$: for each $n$ there is a different rate-distortion optimization which yields $D_n$. Thus, we have a series of optimization problems. It is easy to verify that  $D_n(R)$ is non-increasing in both $n$ and $R$.

The following 
theorem states the relation between $D_n(R)$ and $\text{MER}_\ell^n$.
\begin{theorem}
\label{theorem:rate_distortion_eq_MER}
For a given training set size $n$, for all rates $R\ge I(W;Z^n)$, we have 
\begin{equation*}
D_n(R) = \mathrm{{MER}}_\ell^n.
\end{equation*}
\end{theorem}
Note that since a Markov chain $W - Z^n - \hat{h}$ holds, having $R\ge I(W;Z^n)$ actually removes the constraint on optimization problem (\ref{eq:rate_distortion_base}). Thus, Eq.~(\ref{eq:expected_distortion_eq_MER}) can be used to prove this theorem.

We have seen that $I(W;Z^n)$ appeared in an upper bound on MER in Eq. (\ref{eq:MER_bound_MI_W_dataset}). Combining this fact and Theorem~\ref{theorem:rate_distortion_eq_MER}, for a bounded loss we have
\begin{equation}
\label{eq:D_n_eq_MER}
D_n(I(W;Z^n)) = \text{MER}_\ell^n \le \sqrt{\frac{b^2}{2n}I(W;Z^n)}.
\end{equation}

\subsection{Lower Bound}
\label{sec:lower_bound}
As discussed in Section \ref{sec:appetizer}, to have a standard rate-distortion problem, one can remove the constraint that $\hat{h}$ is generated only using the samples $Z^n$; i.e.
\begin{align}
\label{eq:rate_distortion_lower_bound}
    D^L(R)= \inf_{P_{\hat{h}}^W} \EX[d(W,\hat{h})],
\\
\mathrm{s.t.}\; I(W;\hat{h})\le R. \nonumber
\end{align}
Note that since the feasible set is enlarged, the solution to this minimization will be a lower bound on the optimization of (\ref{eq:rate_distortion_base}):
\begin{equation}
    \label{eq:D_L-LowerBounds-D_n}
    \forall R,\;  \forall n;\;\; D^L(R) \le D_n(R).
\end{equation}
Function $D^L(R)$ is much easier to study than $D_n(R)$, since the corresponding optimization problem \eqref{eq:rate_distortion_lower_bound} is independent of $n$.

In the next sections, we will first derive an upper bound on $D_n(R)$. Then by studying the gap between the upper and lower bounds, we shed light on the behavior of $D_n(R)$.

\subsection{Upper Bound}
\label{sec:upper_bound}
To define the upper bound, we add another constraint to the optimization problem (\ref{eq:rate_distortion_base}): the mutual information between the dataset and the learned model $\hat{h}$ should also be constrained by $R$. More precisely, we define 
\begin{align}
\label{eq:rate_distortion_upper_bound}
    D_n^U(R)= \inf_{P_{\hat{h}}^{Z^n}} \EX[d(W,\hat{h})],
\\
\mathrm{s.t.}\; I(Z^n;\hat{h})\le R. \nonumber
\end{align}
Note that the constraint in (\ref{eq:rate_distortion_upper_bound}) is more strict than the constraint in  (\ref{eq:rate_distortion_base}), since by data processing inequality we have
$$I(Z^n;\hat{h})\le R \Longrightarrow I(W;\hat{h})\le R.$$
Thus, we can write
\begin{equation}
\forall R, \; \forall n; D_n(R) \le D_n^U(R).
\end{equation}

This rate-distortion problem is of interest by itself. Note that an increasingly popular approach in controlling the generalization gap in frequentist setting by information-theoretic tools, is to guarantee that mutual information between dataset and the model is small \cite{XuInformationtheoreticanalysisgeneralization2017,RussoHowmuchdoes2015,RussoControllingbiasadaptive2016, BassilyLearnersthatUse2018}. 
To translate the frequentist setting to the Bayesian setting of our discussion, consider the same form of parametric learning in which the unknown distribution is assumed to be described by the parameter $w$. But no distribution is assumed on the value of $w$, and an algorithm should work for any $w$, in a minimax fashion. Thus, by bounding the mutual information, we mean that for all $w$, $I^w(Z^n;\hat{h})\le R$ and we have $I(Z^n; \hat{h}|W)=\EX_W[I^w(Z^n;\hat{h})]\le R$. On the other hand since there is the Markov chain $W-Z^n-\hat{h}$, we have $I(Z^n; \hat{h}|W)\le I(Z^n; \hat{h})$. Therefore, understanding the effect of the constraint $I(Z^n; \hat{h})\le R$ in the Bayesian setting could be illuminative also for the frequentist setting. In particular, if 
$I(Z^n; \hat{h})\le R$ is satisfied for all $P_W$, we have $I^w(Z^n; \hat{h})\le R; \forall w \in \mathcal{W}$. %

A natural question that should be studied is whether the equality $D_n(I(W;Z^n))\stackrel{?}{=}D^U_n(I(W;Z^n))$ holds. The informal rational behind this question is as follows: 
Intuitively, if we know that only $R=I(W;Z^n)$ nats of information about $W$ is present in the dataset $Z^n$,  it should be possible to just extract those nats without relying more on the dataset. 
Unfortunately, this equality does not hold in general. Actually, often an unbounded $I(\hat{h};Z^n)$ is needed in order to achieve $D_n(R)$. To see this, consider the simple problem where $W\distas{} \mathcal{N}(0,1)$, $Z^n={(Y_i)}_{i=1}^n$, $Y_i\distas{} \mathcal{N}(W,1)$, and $\ell(y,\hat{y})=(y-\hat{y})^2$. In this case, it is easy to verify that there is a unique optimal Bayes decision rule $\psi^*_{Y|Z^n}(z^n)$ (expected value of the posterior distribution of $W$ given $Z^n$), which is a deterministic function of $Z^n$. Thus, while $I(W;Z^n)$ is finite (as we know from well-known results on Gaussian channels), $I(Z^n;\hat{h})$ should be infinite to achieve the best performance.

Despite this unsatisfactory observation, it is actually possible to do quite well with a limited rate if we don't persist in using exactly the optimal decision rule.
This is made precise in the next theorem.
\begin{theorem}
\label{theorem:DUn_bound}
For any bounded loss function $\ell:\mathcal{Y}\times \mathcal{Y} \to [0,b]$, and for all $n \ge 1$, we have 
\begin{align}
\label{eq:D_U_n_eq_MER}
D^U_n(I(W;Z^n)) 
& \le \sqrt{\frac{b^2}{2}I(W;Y|Z^n,X)}\\
&\le \sqrt{\frac{b^2}{2n}I(W;Z^n)}.
\end{align}
\end{theorem}

\subsection{Relation between Lower and Upper Bounds}
\label{sec:relation_between_lower_and_upper_bounds}

The following theorem states the relation between the upper bound $D^U_n(R)$ and the lower bound $D_L(R)$.
\begin{theorem}
\label{theorem:bound_on_difference_of_DUn_and_DL}
For any bounded loss $\ell:\mathcal{Y}\times\mathcal{Y}\to [0,b]$, we have
\begin{equation}
D^U_n(R) \le D^L(R) + \sqrt{\frac{b^2}{2} I(W;\hat{h}_R|Z^n)},
\end{equation}
where the mutual information is based on the distribution $P_{W,\hat{h}_R Z^n}=P_{W}\otimes P^{*W}_{\hat{h}_R}\otimes P_{Z^n}^W$ and $P^{*W}_{\hat{h}_R}$ is a solution to the optimization of $D^L(R)$.
\end{theorem}
This theorem states that to understand the difference between $D^L(R)$ and $D_n^U(R)$, one can solve the optimization associated to $D^L(R)$ to find $P_{\hat{h}_R}^{*W}$. Then, the mutual information $I(W;\hat{h}_R|Z^n)$ controls the gap between the upper bound and lower bound. While this nonasymptotic bound provides an intuitive understanding of the interplay between $D^L(R)$ and $D^U_n(R)$ for all $n$ and $R$, it is hard to be evaluated. 
But as $n\to\infty$, if the posterior is concentrated to the true realization, it is reasonable to expect that $I(W;\hat{h}_R|Z^n)\to 0$ and all of the rate-distortion functions converge. This is made precise in the next theorem.

\begin{theorem}
\label{theorem:convergence_of_rate_distortion_functions}
Suppose the distortion $d(W,\hat{h})$ defined in Eq.~\eqref{eq:d-def} can be represented as a distance $d'(h_W^*, \hat{h})$. Let $W$ and $W'$ be two samples independently generated from $P_W^{Z^n}$. 
If we have $$\lim_{n\to \infty} \EX[d'(h_W^*,h_{W'}^*)] = 0,$$
then 
\begin{equation}
     \forall R\ge0;\; D^L(R)= \lim_{n\to\infty}D_n(R)=\lim_{n\to\infty}D^U_n(R).
\end{equation}
\end{theorem}

Note that the condition $\lim_{n\to \infty} \EX[d'(h_W^*,h_{W'}^*)] = 0$ is usually satisfied as a result of the convergence of the posterior distribution.
In Section \ref{sec:applications}, we will see cases for which the distortion can be represented as a distance.

In Figure~\ref{fig:rate_distortion_functions}, the relation between all the introduced rate-distortion functions is presented. This figure also summarizes some of the presented results.
\begin{figure}[h!]
    \centering
    \includegraphics[scale=0.95]{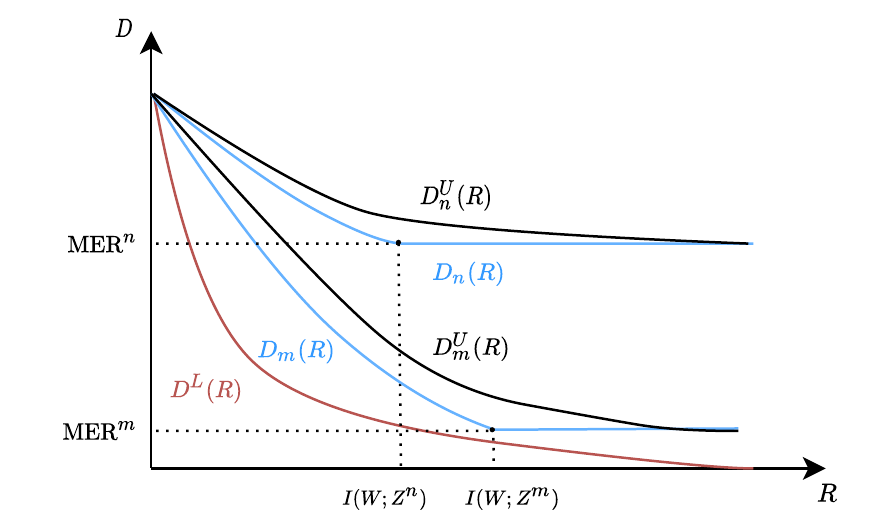}
    \caption{A schematic view of the relation between rate-distortion functions studied in this paper. The original rate-distortion function $D_n(R)$ and its upper bound $D_n^U(R)$ are presented for two sample sizes $n$ and $m$, where $n<m$. The lower bound $D^L(R)$ is also illustrated. As discussed in Theorem~\ref{theorem:rate_distortion_eq_MER}, $D_n(R)$ is equal to $\mathrm{MER}^n$ for $R\ge I(W;Z^n)$. Also note that the upper bound approaches $\mathrm{MER}^n$ as $R\to\infty$. Both $D_n(R)$ and $D_n^U(R)$ approach the $D(R)$ as $n\to\infty$.}
    \label{fig:rate_distortion_functions}
\end{figure}

\section{Applications}
\label{sec:applications}

In the framework developed in previous sections, the distortion function $d(W,\hat{h})$ has a quite general form: it measure the distortion between a variable $W$ and a function $\hat{h}$.  In practice it is usually easier to represent the problem in a way that both input and output are elements of a shared metric space and the distance of that space is the distortion measure. 

In order to achieve this, we first need a lemma which allows us to reformulate the rate-distortion functions.
\begin{lemma}{(Reparameterization Lemma)}
	\label{lemma:reparam}
	Let $d: \mathcal{X} \times \hat{\mathcal{X}}\to \mathbb{R}$ be a distortion function. Assume that there exists mappings $f: \mathcal{X} \to \mathcal{V}$ and $g: \hat{\mathcal{X}} \to \hat{\mathcal{V}}$, and a distortion function $d': \mathcal{V} \times \hat{\mathcal{V}} \to \mathbb{R}$, such that for all $x \in \mathcal{X}$, $\hat{x}\in \hat{\mathcal{X}}$, we have $d(x, \hat{x}) = d'(f(x), g(\hat{x}))$. If $\hat{\mathcal{V}} = f(\hat{\mathcal{X}})$, it follows that
	\begin{align*}
	&\min_{P_{\hat X}^X}\;\; \E_{X, \hat X}\;\big[ d(X, \hat X)\big] \; = &&\min_{P_{\hat V}^V}\;\; \E_{V, \hat V}\;\big[ d'(V, \hat V)\big],
	\\
	&\mathrm{s.t.}\;\;   I(X; \hat X)  \leq R \nonumber && \mathrm{s.t.}\;\;   I(V; \hat V)  \leq R \nonumber
	\end{align*}
	where the second minimization is the rate-distortion function for random variable $V=f(X)$.
\end{lemma}
When the reparameterization is applicable, we might abuse the notation and write $d(V, \hat V)$ instead of $d'(V, \hat V)$.

For quadratic loss, the reparameterization lemma can be used to represent $d(W,\hat{h})$ as a norm on a suitable function space.
Let $\mathcal{Y} \subseteq \mathbb{R}$ and consider $l(y, \hat{y}) = (y - \hat{y})^2$, for all $y, \hat{y}, \in \mathcal{Y}$.  Based on Equation \eqref{eq:d-def},  we have
\begin{align*}
    d(w, \hat{h}) &= \E_{XY}^w\Big[\big|Y- h_w^*(X)\big|^2  - \big|Y - \hat{h}(X)\big|^2\Big]\\
    &= \E_X^w \Big[\big| h_w^*(X) - \hat{h}(X) \big|^2\Big],
\end{align*}
which is the norm of $L^2(P_X)$. Thus using the reparameterization lemma, the rate distortion problems can be restated for the distortion function 
$d(h^*_w, \hat{h})=||h^*_w- \hat{h}||_{L^2(P_X)}$. 

It would be helpful if we could represent the distortion function by a distance on the parameter space, but this is not always possible. To be precise, define the hypothesis class 
\begin{equation}
    \label{eq:hypothesis_class}
    \mathcal{H} = \{h_w(.) = \psi^*_{Y|WX}(w, \cdot) | w\in \mathcal{W} \}
\end{equation}
where $\mathcal{W}$ is the set of all possible $W$s. Note that the optimal function learned from the samples $z^n$, $\psi^*_{Y|Z^nX}(z^n, .)$, does not necessarily lie in $\mathcal{H}$. In other words it might not be parameterizable using $W$. %
But it might still be possible to derive lower bounds %
by projecting on the set $\mathcal{H}$.

Assume that $\mathcal{H}$ is a convex subset of the Hilbert space $L^2(P_X)$. For a given $f \in L^2(P_X)$, define $\mathrm{proj}_\mathcal{H}(f)$ as the projection of $f$ on the convex set $\mathcal{H}\subseteq L^2(P_X)$. As a result of the contraction property of projections on convex sets in Hilbert spaces, we have $d(\mathrm{proj}_\mathcal{H}(\hat{h}), h_w^*) \leq d(\hat{h}, h_w^*)$. Therefore,
\begin{align*}
    D_L(R) \geq\;\; &\min_{P_{\hat{h}}^W}\;\; \E\;\Big[ d\big(h^*_W, \mathrm{proj}_\mathcal{H}(\hat{h})\big)\Big],
    \\
	&\mathrm{s.t.}\;\;   I\big(h_W^*; \mathrm{proj}_\mathcal{H}(\hat{h})\big)  \leq R.
\end{align*}
By the application of Lemma \ref{lemma:reparam}, we arrive at the following lower bound
\begin{align*}
    D_L(R) \geq\;\; &\min_{P_{\hat{W}}^W}\;\; \E\;\Big[ d'\big(W, \hat{W})\big)\Big].
    \\
	&\mathrm{s.t.}\;\;   I\big(W; \hat{W}\big)  \leq R.
\end{align*}
in which  $d'(w, \hat{w}) = d(h^*_w, h_{\hat{w}})$ where $h_{\hat{w}} = \mathrm{proj}_\mathcal{H}(\hat{h})$. This process of projection and reparameterization is summarized in Fig.~\ref{fig:reparam}.
    
\begin{figure}[h!]
    \centering
    \includegraphics[scale=0.75]{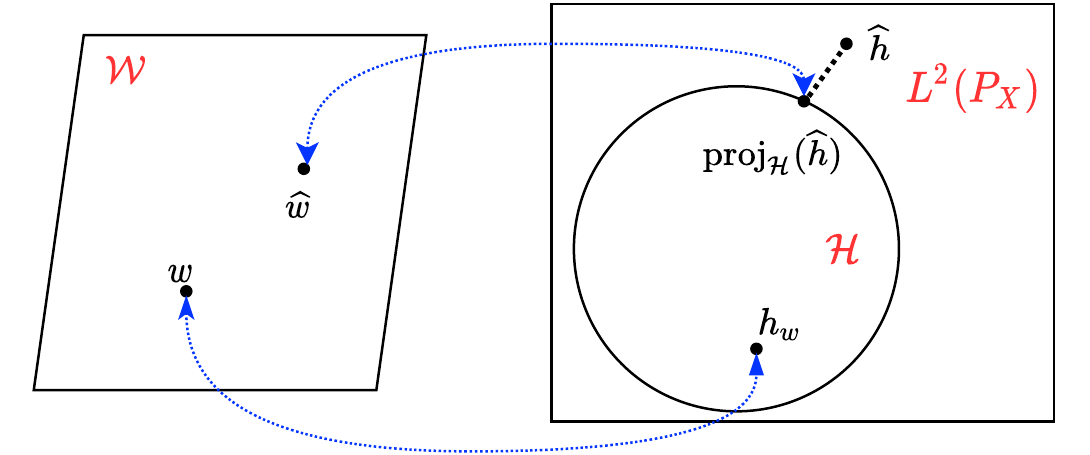}
    \caption{Projection to parametric family of functions.
    In the left, we have the set of all parameters $\mathcal{W}$ and in the right, the set of all functions in $L^2(P_X)$. $\mathcal{H}$ is the set of functions which are associated to a parameter in $\mathcal{W}$. For a function $\hat{h}$ which is not in $\mathcal{H}$, first it is projected to $\mathcal{H}$ using the $L^2(P_X)$ geometry, and the parameter $\hat{w}$ for $\mathrm{proj}_\mathcal{H}(\hat{h})$ is used. Then the reparameterization lemma can be used to restate this refined problem in $\mathcal{W}$.
    }
    \label{fig:reparam}
\end{figure}

Based on Theorem \ref{theorem:rate_distortion_eq_MER} and Equation \eqref{eq:D_L-LowerBounds-D_n}, we know that $\text{MER}_\ell^n \geq D^L(I(W; Z^n))$.  Under the conditions that the space $\mathcal{W}$ is finite-dimensional, and that some regularity conditions on $P_{Z}^w$ hold (see Lemma \ref{lemma:fisher}), we can use this fact and the following  lemma to lower bound MER.

\begin{lemma}(Shannon Lower Bound \cite{shannon1959coding})
    \label{SLB}
	Let $W$ and $\hat{W}$ be random variables taking values in $\R^p$,   $||.||$ be an arbitrary norm on $\R^p$, and $r$ be a positive real number. For any $D\geq 0$, define the rate-distortion function 
	\begin{align*}
    \label{eq:rate_distortion_upper_bound}
    R(D)&=\inf_{P_{\hat W}^W} I(W; \hat W),
    \\
    &\mathrm{s.t.}\;\; \E_{W, \hat W} ||W - \hat W||^r \leq D. \nonumber
    \end{align*}
	We have 
	\begin{equation*}
	R(D) \geq h(W) - \log\Big(V_p \big(\frac{Dre}{p}\Big)^{\frac{p}{r}} \Gamma\big(1+\frac{p}{r}\big)\big),
	\end{equation*}
	in which $h(W)$ is the differential entropy of $W$ and 
	 $V_p$ is the volume of $\{x \in \R^p: ||x||\leq 1\}$.
\end{lemma}
It is known that Shannon Lower Bound is asymptotically tight as $D \to 0$ \cite{wu2020highdiminfo, koch2016shannon}.

Under the conditions of Lemma \ref{lemma:fisher}, we can prove the following lower bound for MER:

\begin{theorem}
	\label{theorem:rate}
	
	Let $\mathcal{W}$ be a $p$-dimensional compact subspace of $\R^p$, the hypothesis set defined in Eq. \eqref{eq:hypothesis_class} be convex, and assume that the regularity conditions of Lemma \ref{lemma:fisher} hold. If there exists a norm $||.||$ such that $||W - \hat{W} ||^2 \leq d'(W, \hat W)$, we have
	\begin{equation*}
	    \mathrm{MER}_\ell^n \geq \frac{p}{n} \cdot \frac{\pi}{\big(V_p \; \Gamma(1+\frac{p}{2})\big)^{\frac{2}{p}}}\exp\Big(-\frac{ \E \log |J_Z^W(W)|}{p} \Big),
	\end{equation*}
	as $n\to \infty$, where $V_p$ is the volume of $\{x \in \R^p: ||x||\leq 1\}$.
\end{theorem}

This theorem formalizes the intuition that a lower MER might be achieved in problems in which the expected Fisher information of $W$ contained in $Z$ is large. Note that this theorem implies  $\text{MER}_\ell^n = \Omega(1/n)$.

If there exists a constant $c$ such that for every $w$, $\big(J_{Z}^W(w)\big)_{ii} \leq c$, we can write
\begin{align*}
    \label{eq:bounded-fisher}
	\E_W \log|J_{Z}^W(W)| &\leq \E_W \log \prod_{i = 1}^{p} \big(J_{Z}^W(W)\big)_{ii} \leq p \log(c),
\end{align*} 
where the first inequality follows from Hadamard's inequality. Also note that given a positive-definite matrix $A$, the volume of the ellipsoid $\{x \in \R^p: ||x||_A\leq 1\}$ is given by $V_p = (\det A)^{-\frac{1}{2}} \frac{\pi^{p/2}}{\Gamma(1+\frac{p}{2})}$. Using these facts, the lower bound of $\Omega(p/n)$ can be obtained for MER,  as stated in the following corollary.

\begin{corollary}
\label{corollary:p-over-n}
Under the conditions of Theorem \ref{theorem:rate}, if $||\cdot|| = ||\cdot||_{A}$ ~ for some positive-definite matrix $A$, and by assuming  that for all $w$; $(J_{Z}^{W}(w))_{ii} \leq c$, we have
\begin{equation*}
    \mathrm{MER}_\ell^n \geq \frac{\gamma p}{n c} =  \Omega\Big(\frac{p}{n}\Big),
\end{equation*}
as $n\to \infty$, in which $\gamma$ is the smallest eigenvalue of $A$.
\end{corollary}

\subsection{Gaussian Location Model}
Let $Y_i = W + V_i; \; \forall i\leq n+1$, where $V_i \sim \mathcal{N}(0, \sigma^2)$ and a  prior $W \sim \mathrm{Unif}(0, 1)$ be assumed on $W$. Given $\{Y_i\}_{i=1}^n$, the goal is to predict $Y_{n+1}$. 
In this problem, we have 
$J_{Z}^W(w) = J_Y^W(w) = \frac{1}{\sigma^2}$. Let $\ell(a, b) = (a-b)^2$ which implies $d(w, \hat{h}) = (w - \hat{h})^2$. The distance $d(w, \hat{h})$ is a norm in the space $\R$, and it satisfies the conditions of Theorem \ref{theorem:rate}. 
Thus,
$\mathrm{MER}_2^n = \Omega(\frac{1}{n})$.

\subsection{Linear Regression}
\label{sec:applications_linear_regression}
Let $Y = W^\top X + \sigma \nu$, where $W \sim P_W$, $X\sim \mathcal{N}(0, \Sigma_X)$, and $\nu \sim \mathcal{N}(0, 1)$, in which  $P_W$ is supported on  a compact subspace $\mathcal{W}$ of $\mathbb{R}^p$. Also assume that $W, X$, and $\nu$ are independent, the matrix $\Sigma_X$ is full-rank, and the space $\mathcal{W}$ is convex. Consider $\ell(a, b) = (a-b)^2$. Note that $h^*_w(x) = w^\top x$ and that the hypothesis class $\mathcal{H} = \{h_w(x) =  w^\top x | w \in \mathcal{W}\}$ is convex. We have
\begin{align*}
	d'(w, \hat{w}) &= d(h_w, h_{\hat w}) \\ 
	&= \E_X [(w^\top X - \hat w^\top X)^2]\\
	&= (w - \hat w)^\top \Sigma_X (w - \hat w),
\end{align*}
meaning that the distance $d'(w, \hat w)$ can be formulated by a norm in the space $\R^p$, i.e.
$d'(w, \hat w) = ||w-\hat w||^2_{\Sigma_X}$. This problem satisfies the assumptions of Corollary \ref{corollary:p-over-n}, and  we have $\mathrm{MER}_\ell^n = \Omega(\frac{p}{n})$.
A similar rate-distortion problem has been studied in the context of compression of linear models in \cite{gao2019rate}.

Note that in this case the loss is unbounded. 
While in Lemma~\ref{lemma:base_bound_for_Bayes_data_process_gap} upper bounds for bounded loss are provided, similar results for the general case of unbounded loss can be derived if the tails of the distribution are suitably controlled; e.g. the distribution is subgaussian (see Theorem~4 of \cite{xu2020minimum}). Moreover, for the case of quadratic loss, if the Gaussian noise $\nu$ is replaced by a bounded random variable, the loss would be bounded and there exists upper bounds with the same rate of $O(\frac{p}{n})$ (see Theorem~3 of \cite{xu2020minimum}).

 \subsection{Certain Classes of Non-Linearities}
 
Fix $w_0 \in \R^p$ and a function $\Phi_{w_0}(\cdot): \mathcal{X} \to \R^p$. Consider the set of nonlinear functions 
$$f(\cdot, w) = f(\cdot, w_0) + \Phi_{w_0}^\top(\cdot) (w - w_0),$$
for $w \in \mathcal{W} \subseteq \mathbb{R}^p$.  The class resembles Neural Tangent Kernels \cite{jacot2018neural}. Assume that $Y = f(X, W) + \sigma\nu$, where $X \sim P_X$, $W \sim P_W$ in which $P_W$ is supported on a compact subset $\mathcal{W}$ of $\mathbb{R}^p$, and  $\nu \sim \mathcal{N}(0, 1)$. Also assume that $\mathcal{W}$ is convex. Consider the loss function $\ell(a, b) = (a-b)^2$. We have $h_w^*(x) = f(x, w)$ and the hypothesis set $\mathcal{H} = \{f(\cdot, w)\, |\, w\in \mathcal{W}\}$ is convex. If  $\E[\Phi_{w_0}(X)\Phi_{w_0}^\top(X)]$ is full-rank, and $\Phi_{w_0}$ is smooth such that the smoothness conditions of Lemma \ref{lemma:fisher} hold, then following the same line of reasoning as the linear regression example, we have $\mathrm{MER}_\ell^n = \Omega(\frac{p}{n})$.

\subsection{Usability for More General Cases}
There are various aspects in Theorem~\ref{theorem:rate} which one should take care of when dealing with more complicated problems. 
For example, while the previous section provided analysis for a very simplified neural network, there are some difficulties to apply such analysis for a more general (Bayesian) neural network. 

One difficulty is the apparent dependence of the bound $\Omega(p/n)$ on $p$ in the over-parameterized regime where we could have $p \gg n$. In particular, in many high dimensional problems, there are just a few dimensions for which the covariance matrix has large eigenvalues; i.e. data mostly resides in a lower dimensional space.
In such scenarios, a more precise treatment is needed. To see that the bound does not depend on $p$, note that the constant hidden in the rate actually depends on the determinant of the covariance matrix, and having small eigenvalues potentially allows one to achieve a smaller MER. 

Moreover, if the Fisher matrix is singular, better (non-singular) parameterization of the problem exists and the Reparameterization Lemma (Lemma~\ref{lemma:reparam}) can be used to take advantage of this fact and then apply the lower bounds. 
The same technique might work for some cases where the mapping is not injective (a requirement which was enforced by the conditions of Lemma~\ref{lemma:fisher}). For example, in neural networks, one source of complexity is that permuting the order of neurons and their corresponding weights in a hidden layer of a fully connected NN does not change the function. It might be possible to define a standard ordering of neurons to tackle this problem, though it might be challenging as other conditions should also be met simultaneously.

\section{Conclusion and Future Work}
\label{sec:conclusion}
In this paper, the recent framework of \cite{xu2020minimum} for studying MER was studied and 
a source coding view on MER was suggested. In this view, the variable $W$ is considered as the input, and the generated hypothesis $\hat{h}$ as the output. A suitable distortion measure $d(w,\hat{h})$ was defined to capture the notion of excess risk. This view was used to find fundamental limits on learning with limited amount of information. Since in learning from dataset $Z^n$, the information is inherently bounded by $I(W;Z^n)$, this view provides a natural methodology to study the limits of learning. %
Using this view, a rate-distortion function $D_n(R)$ was introduced and it was proved that it is equal to MER for large enough $R$.  Then it was demonstrated how $D_n(R)$ is bounded bellow and above by two other rate-distortion functions $D^L(R)$ and $D_n^U(R)$ respectively, which were generated by two natural modifications of the original optimization. The lower bound indicated the limits on the ability of any process generating a hypothesis $\hat{h}$ from $W$ while having a limited rate (not restricted to use a training set). The upper bound indicated the price one should pay if a bound on the $I(Z^n;\hat{h})$ is also enforced, a setting related to model compression. These three rate-distortion functions where studied and various upper and lower bounds on them were derived. In particular, it was demonstrated that (under certain conditions) the lower bound has the right rate matching the upper bound, proving that all of the bounds are order-wise tight, and the rate for $\text{MER}^n_\ell$ is $\Theta(p/n)$. Finally some applications of these results were discussed.

Some problems remained open for future studies.
One of the limitations of the current work, is that      Theorem~\ref{theorem:rate} requires some technical conditions for the $\Omega(p/n)$ to be guaranteed. Analyzing lower rates for more general classes of problems remains an open problem. 
In particular, it is interesting to study MER lower bounds for non-parametric problems. The challenge in this setting is that the underlying results which were used to derive lower bound 
require a finite dimensional parameter space. 
Another interesting direction for future studies is to find conditions which guarantee $O(1/n)$ upper bounds for general bounded (or subgaussian) losses. 
While such rates are well studied from the frequentist standpoint (minimax setting), they are less understood in the Bayesian learning.

\bibliography{main}
\bibliographystyle{icml2021}

\onecolumn
	
\appendix

\pagebreak

\section*{Appendices}

\appendix

\makeatletter
\@addtoreset{theorem}{section}
\@addtoreset{equation}{section}
\@addtoreset{table}{section}
\makeatother

\renewcommand*{\thetheorem}{\Alph{section}.\arabic{theorem}}
\renewcommand*{\thelemma}{\Alph{section}.\arabic{lemma}}

\renewcommand\thetable{\Alph{section}.\arabic{table}}    
\renewcommand\theequation{\Alph{section}.\arabic{equation}}

\section{Proof of Lemma \ref{lemma:base_bound_for_Bayes_data_process_gap}}

We need the following base Lemma, which will be used in many proofs. 

\begin{lemma}[\citealt{XuInformationtheoreticanalysisgeneralization2017}]
		\label{lemma:xu2017main}
		Consider random variables $X$ and $Y$ with joint distribution $P_{XY}$ and
		a function $g:\mathcal{X}\times\mathcal{Y}\to\mathbb{R}$ such that $g(X,Y)$ is $\sigma$-subgaussian under the distribution $P_{\bar{X}\bar{Y}}=P_X\otimes P_Y$\footnote{Recall that a random variable $V$ is $\sigma$-subgaussian if $\log \EX[e^{\lambda(V-\EX[V])}]\le \lambda^2\sigma^2/2$ for all $\lambda\in \mathbb{R}$.}, then
		\begin{equation}
		\abs{ \EX[g(\bar{X},\bar{Y})]-\EX[g(X,Y)]}\le\sqrt{2\sigma^2 I(X;Y)}.
		\end{equation}
	\end{lemma}
Recall that a bounded random variable $L\in[0,b]$ is $b/2$-subgaussian.  

Now, we are ready to prove Lemma~\ref{lemma:base_bound_for_Bayes_data_process_gap}. 
\begin{replemma}{lemma:base_bound_for_Bayes_data_process_gap}
Consider random variables $Y,U$ and $V$ forming Markov chain $Y - U - V$ and an arbitrary non-negative bounded function $\ell:\mathcal{Y}\times \mathcal{Y} \to [0,b]$. We have
\begin{equation}
R_\ell(Y|V) - R_\ell(Y|U) \le \sqrt{\frac{b^2}{2}I(Y;U|V)}.
\end{equation}
\end{replemma}
\begin{proof}
The idea of the proof is similar to the method used in Theorem~4 of \cite{xu2020minimum} which here is presented for general random variables forming a Markov chain. Assume a fixed $v$ is given. Consider a random variable $U'$ which is generated from $P_U^v$. To estimate $Y$ from $v$, a non-optimal approach is to use $\psi^*_{Y|U}(U')$. 
Consider the function $g(U,Y)=\ell(T,\psi^*_{Y|U}(U))$. Now we use the conditioning technique of \cite{hafez2020conditioning} for random variable $V$. Notice that the conditions of Lemma~\ref{lemma:xu2017main} are satisfied for conditional distributions $P_{UY}^v=P_{U}^v\otimes P_Y^U$ and $P_{U'Y}^v=P_{U'}^v\otimes P_Y^v$ for all $v$; thus we have
\begin{equation}
\label{eq:where_we_used_xu2017_in_proof_of_base_lemma}
    \E^v[g(U',Y)-g(U,Y)]\le\sqrt{\frac{b^2}{2}I^v(U;Y)},
\end{equation}
where by taking expectation on $v\sim P_V$ from both sides and using the Jensen's inequality we get
\begin{equation}
    \E[g(U',Y)-g(U,Y)]\le\sqrt{\frac{b^2}{2}I(U;Y|V)}.
\end{equation}
On the other hand, from definition of $\psi^*_{Y|U}$ and $\psi^*_{Y|V}$, we note that $\EX[g(U,Y)]=R_\ell(Y|U)$  and $\EX[g(U',Y)]\ge R_\ell(Y|Z^n)$, which yields
\begin{equation}
    R_\ell(Y|V) - R_\ell(Y|U) \le \EX[g(U',Y)] - \EX[g(U,Y)] \le \sqrt{\frac{b^2}{2}I(Y;U|V)},
\end{equation}
this concludes the proof.
\end{proof}

\section{Conditions for Lemma~\ref{lemma:fisher}}

For the sake of completeness, the technical conditions of Lemma~\ref{lemma:fisher}, originally found in Section 2 of \cite{clarke1994jeffrey}, are presented in this section. Let $\mathcal{W} \subseteq \R^p$ and assume that the densities of $P_{Z}^{W}(.|w)$ exist with respect to the Lebesgue measure.

\begin{enumerate}
	\item [0.] The parameter space $\mathcal{W}$ has a non-void interior and its boundary has a $p$-dimensional Lebesgue measure zero.
	\item [1.a.] (Smoothness) The density $p_{Z}^{W}(z|w)$ is twice continuously differentiable in $w$ for almost every $z$. There also exists $\delta(w)$ such that for every $j, k \in \{1, \dots, p\}$:
	\begin{equation*}
		f(w) = \E \Big[\sup_{w': ||w' - w||\leq \delta(w)}\Big| \frac{\partial^2}{\partial w_j' \partial w_k'} \log p_{Z}^{W}(Z|w')\Big|\Big]
	\end{equation*}
	is finite and continuous.
	
	\item [1.b.] (Existence of Moments) For each $j \in \{1, \dots, d\}$:
	\begin{equation*}
		\E \Big[\Big| \frac{\partial}{\partial w_j} \log p_{Z}^{W}(Z|w)\Big|^{2+\zeta}\Big]
	\end{equation*}
	is finite and continuous, as a function of $w$, for some $\zeta>0$.
	
	\item [2.] Fisher Information matrix and the second derivative of the relative entropy are equal; i.e., define the matrix
	\begin{equation*}
		[I(w)]_{j, k} = \E \Big[\frac{\partial}{\partial w_j} \log p_{Z}^{W}(Z|w) \frac{\partial}{\partial w_k} \log p_{Z}^{W}(Z|w)\Big],
	\end{equation*}
	and 
	\begin{equation*}
	[J(w)]_{j, k} = \Big[\frac{\partial^2}{\partial w'_j\partial w'_k} \mathrm{KL} \Big(P_{Z}^{w} || P_{Z}^{w'}\Big)\Big|_{w' = w}\Big];
	\end{equation*}	
	we have $I(w) = J(w)$. The matrix $I(w)$ is also assumed to be positive definite.
	
	\item [3.] (One to One) For $w\neq w'$, we have $P_{Z}^{w} \neq P_{Z}^{w'}$.
	
	\item [4.] The prior on $W$ is assumed to be continuous. It is also assumed that prior is supported on a compact subset of the interior of $\mathcal{W}$.
\end{enumerate}

\section{Proof of Lemma \ref{lemma:bound_for_Bayes_data_process_gap_for_realizable}}
The proof is similar to the proof of Lemma~\ref{lemma:bound_for_Bayes_data_process_gap_for_realizable}, but instead of using Lemma~\ref{lemma:xu2017main}, we use the following lemma which is based on a symmetrization technique presented by \cite{steinke_reasoning_2020}. 
\begin{lemma}[\citealt{steinke_reasoning_2020}]
\label{lemma:steinke_relaizable_plus}
Let $P_Z$ be a distribution on $\mathcal{Z}$, $P_W^Z$ a conditional distribution to generate $W\in\mathcal{W}$ from $Z$, and  $\ell:\mathcal{W}\times\mathcal{Z}\to [0,1]$ an arbitrary bounded function. Then, we have
\begin{equation}
    \EX_{WZ'}[\ell(W,Z')] \le 2 \EX_{W,Z}[\ell(W,Z) ] + 3 I(W;Z)
\end{equation}
where $Z'\distas{}P_Z$ is an independent copy of $Z$.
\end{lemma}
This lemma is a restating of Theorem~5.8 of \cite{steinke_reasoning_2020} where we used $n=1$ and also upper bounded the ``Conditional Mutual Information (CMI)" by $I(W;Z)$ (see Theorem~2.1 of \cite{haghifam_sharpened_2020} for the proof that CMI is less than $I(W;Z)$). 

Lemma~\ref{lemma:steinke_relaizable_plus} can easily be extended  
to get an alternative bound in the setting of Lemma~\ref{lemma:xu2017main} for non-negative bounded loss. This is summarized in the following corollary.
\begin{corollary}
\label{corollary:restate_steinke_relaizable_plus}
		Consider random variables $X$ and $Y$ with joint distribution $P_{XY}$, and
		a function $g:\mathcal{X}\times\mathcal{Y}\to[0,b]$, then
		\begin{equation}
		\EX[g(\bar{X},\bar{Y})]\le 2\EX[g(X,Y)]+ 3 b I(X;Y).
		\end{equation}
\end{corollary}

\begin{replemma}{lemma:bound_for_Bayes_data_process_gap_for_realizable}
Consider random variables $Y, U$, and $V$ forming Markov chain $Y - U - V$ and an arbitrary non-negative bounded function $\ell:\mathcal{Y}\times \mathcal{Y} \to [0,b]$. We have 
\begin{equation}
R_\ell(Y|V)  \le 2 R_\ell(Y|U)  +3bI(Y;U|V).
\end{equation}
\end{replemma}
\begin{proof}

The proof  
is similar to proof of Lemma \ref{lemma:base_bound_for_Bayes_data_process_gap}, but instead of inequality~(\ref{eq:where_we_used_xu2017_in_proof_of_base_lemma}), we use 
Corollary~\ref{corollary:restate_steinke_relaizable_plus}
to achieve  
\begin{equation}
    E^v[g(U',Y)]\le 2 E^v[g(U,Y)] + 3b I^v(U;Y).
\end{equation}
The rest of the proof follows similarly.
\end{proof}

\section{Tightening the Bounds by Applying the Chaining Technique}
When $\mathcal{Y}$ is a metric space and the loss $\ell:\mathcal{Y}\times \mathcal{Y}\to \mathbb{R}$ is the distance of this space, it is possible to tighten the bounds by applying a chaining argument. The nature of the method is similar to what was used in \cite{AsadiChainingMutualInformation2018} to improve the information theoretic bounds on the generalization gap.
Consider the sequence of functions $(\Pi_i)_{i=i_1}^\infty$ where $i_1$ is the largest integer that satisfies $2^{-(i_1-1)}\ge \text{diam}(\mathcal{Y})$, and for all $i\ge i_1$, $\Pi_i:\mathcal{Y}\to\mathcal{Y}$ is a function satisfying $\ell(y,\Pi_i(y))\le 2^{-i};\forall y\in \mathcal{Y}$. Define $\tilde{Y}_i=\Pi_i(Y)$.
Suppose that for all $y$, $\tilde{y}_{(i_1-1)}=y_0$ for a given $y_0\in \mathcal{Y}$.

Define $L=\ell(Y, \hat{h}(W,X))$ where $\hat{h}(w,x)=\psi^*_{Y|W,X}(w,x)$ for all $W\in \mathcal{W}$ and $X\in \mathcal{X}$.
Similarly define $L'=\ell(Y, \hat{h}(W',X))$ where $W'$ is generated from the posterior $P_W^{Z^n}$. The loss at level $i$ is defined as $L_i=\ell(\tilde{Y}_i, \hat{h}(W,X))$, and similarly $L'_i=\ell(\tilde{Y}_i, \hat{h}(W',X))$.
Define $D=L'-L$. For a fixed integer $M$, we can write
\begin{equation}
D=D_{(i_1-1)} + \sum_{i=i_1}^M D_i - D_{i-1} + D - D_M,
\end{equation}
where $D_i=L'_i-L_i$. By fixing $Z^n=z^n$ and $X=x$ and taking expectation on other random variables we have
\begin{equation}
\EX^{z^nx}[L'-L] = \EX^{z^nx}[D]  =\sum_{i=i_1}^M \EX^{z^nx}[D_i - D_{i-1}] + \EX^{z^nx}[D - D_M],
\end{equation}
where we used $\EX^{z^nx}[D_{(i_1-1)}]=0$, which is true because $\tilde{Y}_{(i_1-1)}=y_0$ (and thus is independent of $W$).
The idea of chaining is to bound each of these terms separately.
Define $E'_i=(L'_i - L'_{i-1})$ and $E_i=(L_i - L_{i-1})$. Note that 
\begin{align}
    D_i-D_{i-1}&=L'_i-L_i - (L'_{i-1}-L_{i-1})
    \\&=(L'_i - L'_{i-1}) - (L_i - L_{i-1})
    \\& = E'_i - E_i.
\end{align}
A bound on random variable $E_i$ (and similarly $E'_i$) can be derived by noting that %
\begin{align}
    L_i - L_{i-1} &= \ell(\tilde{Y}_i, \hat{h}(W,X)) -\ell(\tilde{Y}_{i-1}, \hat{h}(W,X)) \nonumber %
    \\&\le \ell(\tilde{Y}_i, \tilde{Y}_{i-1}) \label{tmp2:2}
    \\&\le \ell(\tilde{Y}_i, Y) +\ell(Y,\tilde{Y}_{i-1}) \label{tmp2:3}
    \\&\le 2^{-i} + 2^{-i+1} \label{tmp2:4}
    \\&= 3\times 2^{-i}.\label{tmp2:5}
\end{align}
In inequalities (\ref{tmp2:2}) and  (\ref{tmp2:3}) the triangle inequality is used as $\ell$ is assumed to be a 
metric. In inequality (\ref{tmp2:4}) the property of the mappings $\pi_i$ and $\pi_{i-1}$ is used (recall that $\ell(y,\Pi_i(y))\le 2^{-i};\forall y\in \mathcal{Y},\forall i$).

Now, Lemma~\ref{lemma:xu2017main} can be used to bound $E'_i - E_i$, since  
$P_W^{z^nx}=P_{W'}^{z^nx}$, $E_i$ is a function of $W$ and $(\tilde{Y}_i,\tilde{Y}_{i-1})$ and  $E'_i$ is a function of $W'$ and $(\tilde{Y}_i,\tilde{Y}_{i-1})$. We have 
\begin{align}
    \EX^{z^nx}[E'_i - E_i]\le \frac{b_i}{2}\sqrt{2I^{z^nx}(W;(\tilde{Y}_i,\tilde{Y}_{i-1}))},
\end{align}
where $b_i=3\times 2^{-i}$. Also note that $\lim_{i\to\infty} E_i = \lim_{i\to\infty} E'_i = 0$ and $\lim_{M\to\infty} D_M =D$. Finally, 
taking expectation with respect to $x$ and $z^n$, we have
\begin{equation}
    \text{MER}^n_\ell \le \EX[L' - L] \le 3 \sum_{i=i_1}^\infty  2^{-i}\sqrt{2I(W;(\tilde{Y}_i,\tilde{Y}_{i-1})|Z^nX)},
\end{equation}
where the first inequality is based on the discussions in proof of Lemma~\ref{lemma:base_bound_for_Bayes_data_process_gap}. If we further assume that $\tilde{Y}_{i-1}$ is a function of $\tilde{Y}_i$ we have
\begin{equation}
\label{eq:MER_chaining_bound}
    \text{MER}^n_\ell \le 3 \sum_{i=i_1}^\infty  2^{-i}\sqrt{2I(W;\tilde{Y}_i|Z^nX)}.
\end{equation}
Note that even if $I(W;\tilde{Y}_i|Z^nX)$ is not finite, the summation in \ref{eq:MER_chaining_bound} could be finite as long as the rate at which the sequence $\left(I(W;\tilde{Y}_i|Z^nX)\right)_{i}$ goes to infinity, is small.

\section{Proof of Theorem~\ref{theorem:rate_distortion_eq_MER}}
\begin{reptheorem}{theorem:rate_distortion_eq_MER}
For a given training set size $n$, for all rates $R\ge I(W;Z^n)$, we have 
\begin{equation*}
D_n(R) = \mathrm{{MER}}_\ell^n.
\end{equation*}
\end{reptheorem}
\begin{proof}
\begin{align*}
D_n(R) &= \inf_{P_{\hat{h}}^{Z^n}} \EX[d(W,\hat{h})]
\\&= \inf_{P_{\hat{h}}^{Z^n}} \EX_{W,\hat{h}XY}[\ell(Y,\hat{h}(X))-\ell(Y,h^*_W(X))]
\\&=-R_\ell(Y|WX)+ \inf_{P_{\hat{h}}^{Z^n}} \EX_{\hat{h}XY}[\ell(Y,\hat{h}(X))]
\\&=-R_\ell(Y|WX)+ \inf_{P_{\hat{h}}^{Z^n}} \EX_{XZ^n\hat{h}}\left[\EX_Y^{XZ^n\hat{h}}[\ell(Y,\hat{h}(X))]\right]
\\&\ge -R_\ell(Y|WX)+ \inf_{P_{\hat{h}}^{Z^n}} \EX_{XZ^n\hat{h}}\left[\psi^*_{Y|Z^nX}(Z^n,X)\right]
\\&= -R_\ell(Y|WX)+  R_\ell(Y|XZ^n)
\\&= \text{MER}_\ell^n
\end{align*}
where all the minimizations are subject to the constraint $I(W;\hat{h})\le R$ and $P_{WZ^n\hat{h}}=P_W\otimes{}P_{Z^n}^W\otimes{}P_{\hat{h}}^{Z^n}$. Now note that when $R\ge I(W;Z^n)$, by data processing inequality we have $I(W;\hat{h})\le  I(W;Z^n) \le R$, thus the constraint on $I(W;\hat{h})$ is automatically satisfied. In this case it is enough to use the deterministic algorithm which chooses $\hat{h}(.)=\psi^*_{Y|Z^nX}(z^n,.)$ for the given $z^n$ to achieve $\text{MER}_\ell^n$.
\end{proof}

\section{Proof of Theorem~\ref{theorem:DUn_bound}}

\begin{reptheorem}{theorem:DUn_bound}
For any bounded loss function $\ell:\mathcal{Y}\times \mathcal{Y} \to [0,b]$, and for all $n \ge 1$, we have 
\begin{align}
D^U_n(I(W;Z^n)) 
& \le \sqrt{\frac{b^2}{2}I(W;Y|Z^n,X)}\\
&\le \sqrt{\frac{b^2}{2n}I(W;Z^n)}.
\end{align}
\end{reptheorem}
\begin{proof}
The construction of the proof is similar to the proof of Lemma~\ref{lemma:base_bound_for_Bayes_data_process_gap}. Consider the (non-optimal) estimator $\hat{h}(x)=\psi^*_{Y|WX}(W',x)$ where $W'$ is a sample from the posterior $P_W^{Z^nX}$. As was proved in Lemma~\ref{lemma:base_bound_for_Bayes_data_process_gap}, we have
\begin{align}
\EX[d(W,\hat{h})] 
&=\EX_{XYZ^n\hat{h}}[\ell(Y,\hat{h}(X))]  - R_{\ell}(Y|WX) \label{eq:EX_d_W_hat_h}
\\&\le \sqrt{\frac{b^2}{2}I(Y;W|Z^nX)} \nonumber
\\&\le \sqrt{\frac{b^2}{2n}I(W;Z^n)}, \nonumber
\end{align}
where the second inequality is proved 
by applying the chain rule on $I(Z^n;W)$ and noting that $I(Z;W|Z^n)$ is a decreasing function of $n$ (see the Proof of Theorem 2 in \cite{xu2020minimum}). 
Now, we just need to show that this process of generating $\hat{h}$ provides a feasible point for the optimization of $D_n^U(R)$ for $R=I(W;Z^n)$ in Eq.~(\ref{eq:rate_distortion_upper_bound}). To see this, note that $P_{WZ^n}=P_{W'Z^n}$ and we have $I(W;Z^n)=I(W';Z^n)$.
Now, since the Markov chain $W-Z^n-W'-\hat{h}$ holds, by using the data processing inequality, we have 
$$
I(Z^n;\hat{h})\le I(Z^n;W') = I(W;Z^n)=R.
$$
Thus, we have found a feasible candidate for solving the minimization, and $\EX[d(W,\hat{h})]$ in Eq.~(\ref{eq:EX_d_W_hat_h}) provides an upper bound on the real solution $D^U_n(I(W;Z^n))$.
\end{proof}

Note that in the proof of Theorem \ref{theorem:DUn_bound}, the proposed solution to find $\hat{h}$ was also a parametric approach; i.e., $\hat{h}\in\mathcal{H}\triangleq\left\{ \psi^*_{Y|WX}(w,.) | w\in\mathcal{W} \right\}$. So we actually found an upper bound which would still work if there was another constraint on optimization of $D^U_n(R)$ (Eq.~(\ref{eq:rate_distortion_upper_bound})) which restricted $\hat{h}\in\mathcal{H}$.

\section{Proof of Theorem~\ref{theorem:bound_on_difference_of_DUn_and_DL}}

\begin{reptheorem}{theorem:bound_on_difference_of_DUn_and_DL}
For any bounded loss $\ell:\mathcal{Y}\times\mathcal{Y}\to [0,b]$, we have
\begin{equation}
D^U_n(R) \le D^L(R) + \sqrt{\frac{b^2}{2} I(W;\hat{h}_R|Z^n)},
\end{equation}
where the mutual information is based on the distribution $P_{W,\hat{h}_R Z^n}=P_{W}\otimes P^{*W}_{\hat{h}_R}\otimes P_{Z^n}^W$ and $P^{*W}_{\hat{h}_R}$ is a solution to the optimization of $D^L(R)$.
\end{reptheorem}
\begin{proof}
Fix the rate $R$. Consider random variable $\hat{h}_R$ which its joint distribution with $W$ is $P_{W,\hat{h}_R}=P_{W}\otimes P^{*W}_{\hat{h}_R}$. Now let
parameter 
$W' \sim P_W^{Z^n}$ be an independent copy of $W$ for the given $Z^n$
and consider another random variable $\hat{h}'_R$ generated from $W'$ using the conditional distribution $P_{\hat{h}'_R}^{W'}=P^{*W}_{\hat{h}_R}$.

Note that $P_{\hat{h}'_R}^{Z^n}$
is in the feasible set of optimization of $D^U_n(R)$ since 
$$I(Z^n;\hat{h}'_R)\le I(W';\hat{h}'_R) = I(W;\hat{h}_R) \le R,$$
where the equality is due to the fact that $P_{W'\hat{h}'_R}=P_{W\hat{h}_R}$ and the final inequality is true since $P^{W}_{\hat{h}_R}=P^{*W}_{\hat{h}_R}$ is a solution for $D^L(R)$. 

Thus, by bounding $E[d(W,\hat{h}'_R)]$ we can find an upper bound on $D^U_n(R)$.
This is done by using Lemma~\ref{lemma:xu2017main} for the function $d(W,\hat{h})$. Fix $z^n$, and note that  $P_{W\hat{h}'_R}^{Z^n}=P_{W}^{Z^n}\otimes P_{\hat{h}_R}^{Z^n}$. Thus, the conditions of Lemma~\ref{lemma:xu2017main} are satisfied and we have  
$$\EX^{z^n}[d(W,\hat{h}'_R)] - \EX^{z^n}[d(W,\hat{h}_R)]\le \sqrt{\frac{b^2}{2}I^{z^n}(W;\hat{h}_R)}.$$
By taking expectation from both sides and using the concavity of the square root function, we get
$$\EX[d(W,\hat{h}'_R)] - \EX[d(W,\hat{h}_R)]\le \sqrt{\frac{b^2}{2}I(W;\hat{h}_R|Z^n)}.$$
Finally, since $P_{W\hat{h}_R}$ provides a solution for $D^L(R)$ and $P_{W\hat{h}'_R}$ provides an upper bound for $D^U_n(R)$, we have
$$D^U_n(R)-D^L(R) \le \EX[d(W,\hat{h}'_R)] - \EX[d(W,\hat{h}_R)]\le \sqrt{\frac{b^2}{2}I(W;\hat{h}_R|Z^n)},$$
which concludes the proof.
\end{proof}

\section{Proof of Theorem~\ref{theorem:convergence_of_rate_distortion_functions}}
To prove this theorem, it is easier to describe the achievable rate $R$ as a function of distortion $D$.
So instead of $D_n(R)$, $D^L(R)$ and $D^U_n(R)$, consider optimizations
\begin{align}
\label{eq:optimization_R_n_D}
    R_n(D)= \inf_{P_{\hat{h}}^{Z^n}} I(W;\hat{h}),
\\
\mathrm{s.t.}\;  \EX[d(W,\hat{h})]\le D, \nonumber
\end{align}
and
\begin{align}
\label{eq:optimization_R_L_D}
    R^L(D)= \inf_{P_{\hat{h}}^{W}} I(W;\hat{h}),
\\
\mathrm{s.t.}\;  \EX[d(W,\hat{h})]\le D, \nonumber
\end{align}
and 
\begin{align}
\label{eq:optimization_R_U_n_D}
    R^U_n(D)= \inf_{P_{\hat{h}}^{Z^n}} I(Z^n;\hat{h}),
\\
\mathrm{s.t.}\;  \EX[d(W,\hat{h})]\le D. \nonumber
\end{align}
The following theorem describes a relation between $R^L(D)$ and $R^U_n(D)$.

\begin{theorem}
\label{theorem:bound_on_difference_of_RUn_and_RL}
Suppose the distortion $d(W,\hat{h})$ defined in Eq.~\eqref{eq:d-def} can be represented as a distance $d'(h_W^*, \hat{h})$. Let $W$ and $W'$ be two samples independently generated from $P_W^{Z^n}$.
We have
\begin{equation*}
R^U_n(D') \le R^L(D) - I(W;\hat{h}_D|Z^n),
\end{equation*}
where $D'=D+\EX[d(W,W'_n)]$ and the conditional mutual information is based on the  distribution $P_{WZ^n\hat{h}_D}=P_{W}\otimes P_{Z^n}^W \otimes P^{*W}_{\hat{h}_D}$ and $P^{*W}_{\hat{h}_D}$ is a solution to the optimization of $R^L(D)$.
\end{theorem}
\begin{proof}
Consider the random variable $\hat{h}'_D$ generated from $W'$ using the conditional distribution $P_{\hat{h}'_D}^{W'}=P_{\hat{h}_D}^{*W}$. We have the Markov chain $\hat{h}_D - W - Z^n - W' - \hat{h}'_D$.
We can write
\begin{align}
\EX[d(W,\hat{h}'_D)] &= 
\EX[d'(h_W^*, \hat{h}'_D)] 
\nonumber\\&\le  \EX[d'(h_W^*, h_{W'}^*)]+\EX[d'(h_{W'}^*, \hat{h}'_D)]
\nonumber\\&\le  \EX[d'(h_W^*, h_{W'}^*)]+D.
\end{align}
In the first inequality we used the triangle inequality which holds since $d'$ is a distance. The second inequality is due to the fact that $P_{W'\hat{h}'_D}=P_{W\hat{h}_D}=P_W\otimes P_{\hat{h}_D}^{*W}$, and that $P_{\hat{h}_D}^{*W}$ is in the feasible set of optimization \eqref{eq:optimization_R_L_D}; i.e., $\EX[d(W,\hat{h}_D)]=\EX[d'(h_{W}^*, \hat{h}_D)]\le D$. If we define $D'=\EX[d'(h_W^*, h_{W'}^*)]+D$, we see that the process of generating $\hat{h}'$ meets the constraint of minimization describing $R^U_n(D')$ (Eq.~\ref{eq:optimization_R_U_n_D}), thus $I(Z^n;\hat{h}'_D)$ provides an upper bound on $R^U_n(D')$. 
Thus, we have
\begin{align}
R^U_n(D') &\le I(Z^n;\hat{h}'_D)  \nonumber
\\&= I(W';\hat{h}'_D) - I(W';\hat{h}'_D|Z^n)
\label{tmp_1:2}
\\&=I(W;\hat{h}_D) - I(W;\hat{h}_D|Z^n)
\label{tmp_1:3}
\\&= R^L(D)-I(W;\hat{h}_D|Z^n),
\label{tmp_1:4}
\end{align}
where the equality~(\ref{tmp_1:2}) is based on     
$I(W';\hat{h}'_D)=I(W',Z^n;\hat{h}'_D)=I(Z^n;\hat{h}'_D)+I(W';\hat{h}'_D|Z^n)$, inequality~(\ref{tmp_1:3}) is based on the fact that $P_{Z^nW\hat{h}_D}=P_{Z^nW'\hat{h}'_D}$ and finally equality~(\ref{tmp_1:4})  is true since $P_{W,\hat{h}_D}=P_{W}\otimes P^{*W}_{\hat{h}_D}$ and $P^{*W}_{\hat{h}_D}$ is a solution to the optimization of $R^L(D)$.
\end{proof}
Now we can prove Theorem~\ref{theorem:convergence_of_rate_distortion_functions}.
\begin{reptheorem}{theorem:convergence_of_rate_distortion_functions}
Suppose the distortion $d(W,\hat{h})$ defined in Eq.~\eqref{eq:d-def} can be represented as a distance $d'(h_W^*, \hat{h})$. Let $W$ and $W'$ be two samples independently generated from $P_W^{Z^n}$. 
If we have $$\lim_{n\to \infty} \EX[d'(h_W^*,h_{W'}^*)] = 0,$$
then 
\begin{equation}
\label{eq:limits_converge}
     \forall R\ge0;\; D^L(R)= \lim_{n\to\infty}D_n(R)=\lim_{n\to\infty}D^U_n(R).
\end{equation}
\end{reptheorem}
\begin{proof}
Note that in the rate distortion theory, functions $R^L(D)$ and $D^L(R)$ are decreasing and convex functions both describing the boundary of the achievable $(R,D)$ pairs  
(see Chapter 10 of \cite{CoverElementsinformationtheory2012}). 
As they are different representations of the same entity, the achievability  results can be derived based on either $R(D)$ or $D(R)$.
The same is true for $R^U_n(D)$ and $D^U_n(R)$, as well as $R_n(D)$ and $D_n(R)$. 
In order to prove the theorem, we need to show that the set of achievable $(R,D)$ pairs coincide for all of three minimization problems as $n\to\infty$.
As such, it is enough to equivalently show that $\lim_{n\to\infty} R^U_n(D) =\lim_{n\to\infty} R_n(D) = R^L(D)$.

Define $d_{\text{min}}=\lim_{n\to\infty} D^L(R)$.
Define $d=D^L(R)$. For now we assume $d>d_{\text{min}}$, later we will also address the case $d=d_{\text{min}}$\footnote{If there is no achievable $d$ that satisfies $d>d_{\text{min}}$, it means that $D^L(R)=d_{\text{min}};\forall R$ and  Eq.~\eqref{eq:limits_converge} is trivially satisfied since for all $n$ and $R$,
$D^L(R)=D^U_n(R)=D_n(R)$ holds. To see this, note that at $R=0$ all the optimizations are equivalent; i.e. $D^L(0)=D^U_n(0)=D_n(0)$ and $D^U_n(R)$ and $D_n(R)$ are decreasing functions lower bounded by $D^L(R)=D^L(0)=d_{\text{min}}$.}. We have $R^L(d)= R \le\infty$. 
Define $\delta_n=\EX[d'(h_W^*,h_{W'}^*)]$.
By using Theorem~\ref{theorem:bound_on_difference_of_RUn_and_RL}, we have
\begin{align}
\lim_{n\to\infty} R^U_n(d) 
&\le \lim_{n\to\infty} R^L(d-\delta_n) - I(W;\hat{h}_{d-\delta_n}|Z^n)
\\&=R^L(d)- \lim_{n\to\infty} I(W;\hat{h}_{d-\delta_n}|Z^n) \label{eq:tmp3:2}
\\&\le R^L(d),
\end{align}
where in the equality \eqref{eq:tmp3:2} the continuity of $R^L(d)$ (for $d > d_{\text{min}}$) is used (see Corollary~9.4.2 of \cite{gallager1968information}) along with the assumption that $\lim_{n\to \infty} \delta_n = 0$. On the other hand, we know that $R^U_n(d)\ge R_n(d) \ge R^L(d); \forall n$. Thus, we have $$\lim_{n\to\infty} R^U_n(d) = \lim_{n\to\infty} R_n(d) =  R^L(d); \forall d>d_{\text{min}}.$$ 
Equivalently, we have
$$\lim_{n\to\infty} D^U_n(r) = \lim_{n\to\infty} D_n(r) =  D^L(r),$$
where $r$ is a rate satisfying $D(r)>d_{\text{min}}$. 
Now, to complete the proof, we just need to handle the case where there exists $r_{\text{max}}$ such that $D(r_{\text{max}})=d_{\text{min}}$, and show
$$\lim_{n\to\infty} D^U_n(r_{\text{max}}) = \lim_{n\to\infty} D_n(r_{\text{max}}) =  D^L(r_{\text{max}}),$$
which is true by monotone convergence since $D^U_n(r)$ and $D_n(r)$ are decreasing functions converging to $D^L(r);\forall r<r_{\text{max}}$ and $D^L(r)$ is continuous.
\end{proof}

\section{Proof of Lemma~\ref{lemma:reparam}}
\begin{replemma}{lemma:reparam}{(Reparameterization Lemma)}
	Let $d: \mathcal{X} \times \hat{\mathcal{X}}\to \mathbb{R}$ be a distortion function. Assume that there exists mappings $f: \mathcal{X} \to \mathcal{V}$ and $g: \hat{\mathcal{X}} \to \hat{\mathcal{V}}$, and a distortion function $d': \mathcal{V} \times \hat{\mathcal{V}} \to \mathbb{R}$, such that for all $x \in \mathcal{X}$, $\hat{x}\in \hat{\mathcal{X}}$, we have $d(x, \hat{x}) = d'(f(x), g(\hat{x}))$. If $\hat{\mathcal{V}} = f(\hat{\mathcal{X}})$, it follows that
	\begin{alignat*}{2}
    &\min_{P_{\hat X}^X}\; \E_{X, \hat X}\;\big[ d(X, \hat X)\big] \;\; =  && \;\; \min_{P_{\hat V}^V}\; \E_{V, \hat V}\;\big[ d'(V, \hat V)\big],
    \\
    &\mathrm{s.t.}\;   I(X; \hat X)  \leq R \nonumber && \mathrm{s.t.}  \; I(V; \hat V)  \leq R \nonumber
    \end{alignat*}
	where the second minimization is the rate-distortion function for random variable $V=f(X)$.
\end{replemma}
\begin{proof}
Define $D_1(R)$ and $D_2(R)$ as follows:
\begin{equation}
	D_1(R) = \min_{P_{\hat X}^X} \E[d(X, \hat X)] \;\; \text{s.t. \;\;} I(X; \hat{X}) \leq R,
\end{equation}
and
\begin{equation}
D_2(R) = \min_{P_{\hat V}^V} \E[d'
(V, \hat V)] \;\; \text{s.t. \;\;} I(V; \hat{V}) \leq R.
\end{equation}

Now, let $P_{\hat X}^{* X}$ be a solution for $D_1(R)$. We have $\E[d(X, \hat{X})] = \E[d'\big(f(X), g(\hat{X})\big)]$ where the expectations are with respect to $P_{X\hat{X}} = P_X \otimes P_{\hat{X}}^{* X}$. Note that the Markov chain $f(X) - X - \hat{X} -   g(\hat X)$ holds. Hence, $I(f(X); g(\hat X)) \leq I(X; \hat{X}) \leq R$. Thus, $P_{\hat X}^{* X}$ provides a feasible candidate for the second optimization, which implies that $D_1(R) \geq D_2(R)$.

So it remains to prove that $D_1(R) \leq D_2(R)$. Let $P_V$ be the distribution of $V = f(X)$ where $X\sim P_X$, and $P_{\hat V}^{* V}$ be the solution for $D_2(R)$. We have 
\begin{equation}
	\E[d'(V, \hat{V})] = \E[d(f^\dagger(V), g^\dagger(\hat{V}))],
\end{equation}
where the expectations are with respect to $(V, \hat V) \sim P_V \otimes P_{\hat V}^{* V}$, and $f^\dagger$ and $g^\dagger$ are functions for which 
$\hat{V} =  g(g^\dagger (\hat{V}))$ and $V =  f(f^\dagger (V))$ almost surely. 
The function ${f}^\dagger$ exists by construction of $P_V$. To ensure the existence of $g^\dagger$, we need to assume that $\hat{\mathcal{V}} = f(\hat {\mathcal{X}})$. 
Now, let $\hat{X} = g^\dagger (\hat{V})$. Note that $V = f(X)$, and that the Markov chain $X - V - \hat{V} - g^\dagger(\hat V) = \hat X$ holds. Hence, $I(X; \hat{X}) \leq I(V; \hat{V}) \leq R$. Which means that $P_{\hat V}^{* V}$ provides a feasible candidate for the first optimization.
On the other hand, we have
\begin{equation}
	\E[d'(V, \hat{V})] = \E[d'(f(X), g(\hat X))] = E[d(X, \hat{X})],
\end{equation}
thus $D_2(R) \leq D_1(R)$, which concludes the proof. 
\end{proof}

\section{Proof of Theorem~\ref{theorem:rate}}
\begin{reptheorem}{theorem:rate}
	Let $\mathcal{W}$ be a $p$-dimensional compact subspace of $\R^p$, the hypothesis set defined in Eq. \eqref{eq:hypothesis_class} be convex, and assume that the regularity conditions of Lemma \ref{lemma:fisher} hold. If there exists a norm $||.||$ such that $||W - \hat{W} ||^2 \leq d'(W, \hat W)$, we have
	\begin{equation*}
	    \mathrm{MER}_\ell^n \geq \frac{p}{n} \cdot \frac{\pi}{\big(V_p \; \Gamma(1+\frac{p}{2})\big)^{\frac{2}{p}}}\exp\Big(-\frac{ \E \log |J_Z^W(W)|}{p} \Big),
	\end{equation*}
	as $n\to \infty$, where $V_p$ is the volume of $\{x \in \R^p: ||x||\leq 1\}$.
\end{reptheorem}
\begin{proof}

Define 
\begin{equation}
	R^L(D) = \inf_{P_{\hat W}^W} I(W; \hat W) \;\; \text{s.t.} \;\; \E[d'(W, \hat W)] \leq D,
\end{equation}
which is the inverse of $D^L(R)$ (see Lemma 4.1.2 of \cite{Gray}) and reparameterization lemma (Lemma \ref{lemma:reparam}) is used to define the problem in terms of $d'(W,\hat W)$.
It is assumed that for all $w, \hat w$, the distortion measure $d'(w, \hat w)$ can be lower bounded by $||w-w'||^2$, where $||.||$ is a norm in $\mathbb{R}^p$. Hence, for the rate-distortion function $R'_L(D)$, defined as
\begin{equation}
{R'}_L(D) = \inf_{P_{\hat W}^W} I(W; \hat W) \;\; \text{s.t.} \;\; \E[||W - \hat W||^2] \leq D,
\end{equation}
we have $R'_L(D) \leq R^L(D)$. We can use Shannon Lower Bound (Lemma~\ref{SLB}), with $r=2$, to further lower bound $R'_L(D)$:
\begin{equation*}
	R^L(D) \geq R'_L(D) \geq h(W) - \log\Big(V_p \big(\frac{2De}{p}\big)^{\frac{p}{2}} \Gamma(1 + \frac{p}{2})\Big).
\end{equation*}
Thus, for $D^L(R)$, we can write
\begin{equation}
	\label{eq:DofR}
	D^L(R) \geq \frac{p}{2eC_p} \exp\bigg(\frac{2h(W) - 2R}{p}\bigg),
\end{equation}
in which ${C_p = \Big(V_p\, \Gamma(1+\frac{p}{2})\Big)^{\frac{2}{p}}}$. Minimum Excess Risk and $D^L(R)$ are as %
\begin{equation}
	\mathrm{MER}_\ell^n = D_n(I(W; Z^n)) \geq D^L(I(W; Z^n)),
\end{equation}
where the equality is due to Theorem~\ref{theorem:rate_distortion_eq_MER}.
To derive an asymptotic lower bound on MER as $n$ goes to infinity, we can use Lemma~\ref{lemma:fisher} to lower bound $D^L(I(W; Z^n)$. Based on this lemma, we have
\begin{equation}
    \label{eq:asymp_mutual_W_Zn}
	I(W; Z^n) = \frac{p}{2} \log\big(\frac{n}{2\pi e}\big) + h(W) + \frac{\E\big[\log | J_Z^W(W)|\big]}{2} + o(1) \;\;\;\;\; \text{as}\;\; n\to\infty.
\end{equation}
By substituting $I(W; Z^n)$ from \eqref{eq:asymp_mutual_W_Zn} in Eq. \eqref{eq:DofR}, we arrive at the following lower bound
\begin{equation*}
	\text{MER}_\ell^n \geq \frac{p}{n} \cdot \frac{\pi}{\big(V_p \; \Gamma(1+\frac{p}{2})\big)^{\frac{2}{p}}}\exp\Big(-\frac{ \E \log |J_Z^W(W)| + o(1)}{p} \Big).
\end{equation*}
Note that the differential entropy $h(W)$ is vanished.
\end{proof}

\begin{repcorollary}{corollary:p-over-n}
Under the conditions of Theorem \ref{theorem:rate}, if $||\cdot|| = ||\cdot||_{A}$ ~ for some positive-definite matrix $A$, and by assuming  that for all $w$; $(J_{Z}^{W}(w))_{ii} \leq c$, we have
\begin{equation*}
    \mathrm{MER}_\ell^n \geq \frac{\gamma p}{n c} =  \Omega\Big(\frac{p}{n}\Big),
\end{equation*}
as $n\to \infty$, in which $\gamma$ is the smallest eigenvalue of $A$.
\end{repcorollary}
\begin{proof}
In this case, we have $V_p = (\det A)^{-\frac{1}{2}}\frac{\pi^{\frac{p}{2}}}{\Gamma(1+\frac{p}{2})}$. Note that  $V_p \leq (\frac{\pi}{\gamma})^{\frac{p}{2}}
    \Gamma^{-1}(1+\frac{p}{2})$.
Besides, $\E_W \log |J_Z^W(W)| \leq p \log(c)$. Hence, based on Theorem \ref{theorem:rate}, as $n\to \infty$ we have 
\begin{align}
    \text{MER}_\ell^n &\geq \frac{p}{n} \cdot \frac{\pi}{\big(V_p \; \Gamma(1+\frac{p}{2})\big)^{\frac{2}{p}}}\exp\Big(-\frac{ \E \log |J_Z^W(W)|}{p} \Big) \geq \frac{p\gamma}{n}\cdot \frac{1}{c},
\end{align}
concluding the proof.
\end{proof}

\end{document}